\documentclass[letterpaper]{article} 
\usepackage{aaai24}  
\usepackage{times}  
\usepackage{helvet}  
\usepackage{courier}  
\usepackage[hyphens]{url}  
\usepackage{graphicx} 
\usepackage{natbib}  
\usepackage{caption} 
\usepackage{algorithm}
\usepackage{algorithmic}

\usepackage{amsmath,amsfonts,amssymb,dsfont,bm,mathtools,amsthm}
\usepackage{pgfplots}
\usepackage{tikz} 

\usepackage{subcaption}
\usepackage{bbold} 
\usepackage{verbatim}

\usepackage{thmtools,thm-restate}

\usepackage{newfloat}
\usepackage{listings}
\DeclareCaptionStyle{ruled}{labelfont=normalfont,labelsep=colon,strut=off} 
\floatstyle{ruled}
\newfloat{listing}{tb}{lst}{}
\floatname{listing}{Listing}
\newcommand{\indep}{\rotatebox[origin=c]{90}{$\models$}}
\newcommand{\st}{\text{ such that }}
\newcommand{\eg}{\textit{e.g.}}
\newcommand{\ie}{\textit{i.e.}}

\newtheorem{definition}{Definition}
\newtheorem{assumption}{Assumption}
\newtheorem{example}{Example}

\newenvironment{sproof}{%
  \proof}{\endproof}
\usepackage{lipsum}
\newcommand{\myscale}{0.162}
\newcommand{\mysecondscale}{.95}

\title{Identifiability of Direct Effects from Summary Causal Graphs
\\[1ex] \large including appendix
}
\author{
    Simon Ferreira\textsuperscript{\rm 1,2},
    Charles K. Assaad\textsuperscript{\rm 1}
}
\affiliations{
    \textsuperscript{\rm 1}EasyVista\\

    \textsuperscript{\rm 2} ENS de Lyon\\
    simon.ferreira@ens-lyon.fr, cassaad@easyvista.com
}

\usepackage{bibentry}

\begin{document}

\maketitle

\begin{abstract}
Dynamic structural causal models (SCMs) are a powerful framework for reasoning in dynamic systems about direct effects which measure how a change in one variable affects another variable while holding all other variables constant. The causal relations in a dynamic structural causal model can be qualitatively represented with an acyclic full-time causal graph. Assuming linearity and no hidden confounding and given the full-time causal graph, the direct causal effect is always identifiable. 
However, in many application such a graph is not available for various reasons but nevertheless experts have access to the summary causal graph of the full-time causal graph which represents causal relations between time series while omitting temporal information and allowing cycles.
This paper presents a complete identifiability result which characterizes all cases for which the direct effect is graphically identifiable from a summary causal graph and gives two sound finite adjustment sets  that can be used to estimate the direct effect whenever it is identifiable.
\end{abstract}

\section{Introduction}

Structural causal models (SCMs) are a powerful framework for representing and reasoning about causal relations between variables with a long history in many fields such as genetics~\citep{Wright_1920,Wright_1921}, econometrics~\citep{Haavelmo_1943}, social sciences~\citep{Duncan_1975,Goldberger_1972}, epidemiology~\citep{Hernan_2023}, and artificial intelligence~\citep{Pearl_2000}.
In particular, SCMs are useful for reasoning about direct effects which  measure how a change in one variable affects another variable while holding all other variables constant~\citep{Pearl_2012}. 
The identification and estimation of direct effects are important in many application, \eg, 
epidemiologists are interested in measuring how smoking affects lung cancer risk without being mediated by genetic susceptibility~\citep{Zhou_2021};
ecologists are usually focus on understanding direct effects such as competition, herbivory, and predation~\citep{Connell_1961};
IT experts can localize root causes of system failures by comparing the direct impact of different components on each other before and after the failure~\citep{Assaad_2023}.

In the framework of (non-dynamic) SCMs, assuming linearity and no hidden confounding and given a causal graph which qualitatively represents causal relations between different variables, the direct effect between two variables is always identifiable and there exists a complete graphical tool, called the single-door criterion~\citep{Pearl_1998,Spirtes_1998,Pearl_2000} that finds all possible adjustment sets that allow to estimate the direct effect from data.
These results are directly applicable in dynamic SCMs~\citep{Rubenstein_2018} given an acyclic full-time causal graph\textemdash which qualitatively represents all causal relations between different temporal instants of the dynamic SCM\textemdash and assuming consistency throughout time.
However, in many dynamic systems, experts have difficulties in building a full-time causal graph~\citep{Ait_Bachir_2023}, while they can usually build a summary causal graph \citep{Assaad_2022_a} which is an abstraction of the full-time causal graph where temporal information is omitted and cycles are allowed.
%
%
So far, the problem of identifying direct effects has solely been tackled for summary causal graphs with no cycles exceeding a size of $2$~\cite{Assaad_2023}. Specifically, it has been shown that in the absence of cycles greater than $2$, direct effects are always identifiable, and in these cases, an adjustment set has been proposed for estimating the direct effect from data. However, in many applications, there exist cycles with a size greater than $2$~\cite{Veilleux_1979,Staplin_2016}.

In this work, we focus on the identifiability of direct effects from summary causal graphs without restricting the size of cycles.
Our main contribution is twofold. First, we give a complete identifiability result which characterizes all cases for which a direct  effect is graphically identifiable from a summary causal graph. Then, we present two finite adjustment sets  that can be used to estimate the direct  effect from data whenever it is identifiable.

The remainder of the paper is organized as follows: in the next section, we give the definitions of direct effect and summary causal graph and recall all necessary graphical preliminaries.  
In the section that follows it, we present the complete identifiability result in addition to a weaker but interesting result that is implied by the complete identifiability result. 
Then, in another section, we provide two finite adjustment sets that can be used to estimate the direct  effect from data whenever it is identifiable.  
Finally, in the last section, we conclude the paper while elaborating on promising directions for further research. 

\section{Problem Setup}
\label{sec:setup}

In this section, we first introduce some terminology, tools, and assumptions which are standard for the major part. Then, we formalize the problem we are going to solve.
Without any opposite mention, in the following, 
a capital letter correspond to either to a variable or a vertex and a calligraphic letter correspond to a set with the exception of the set of natural numbers and the set of all integers which are respectivly represented by $\mathbb{N}$ and $\mathbb{Z}$.

In this work, we consider that a dynamic system can be represented by a linear dynamic SCM.

\begin{definition}[Linear Dynamic SCM]
\label{def:SCM}
    Considering a finite set of observed times series $\mathcal{V}$, a \emph{linear dynamic SCM} is a set of equations in which each instant $t \in \mathbb{Z}$ of a time series (\eg, $Y_t$) is defined as a linear function of past instants of itself (\eg, $Y_{t-\gamma},~\gamma>0$), past or present instants of other times series (\eg, $X_{t-\gamma},~X\neq Y,~\gamma\geq0$) and of some unobserved noise (\eg, $\xi_{Y_t}$):
    \begin{equation*}
        \label{eq:linearSCM}
        Y_t := \sum_{\gamma > 0}\alpha_{Y_{t-\gamma},Y_t}*Y_{t-\gamma} + \sum_{X\neq Y,~\gamma \geq 0}\alpha_{X_{t-\gamma},Y_t}*X_{t-\gamma} + \xi_{Y_t},
    \end{equation*}
    where any coefficient $\alpha$ can be zero and for each $Y\in \mathcal{V}$ the noises $\{ \xi_{Y_t} , t\in \mathbb{Z}\}$ are identically distributed.
\end{definition}

Note that in nonlinear SCMs, the direct effect is not uniquely defined, potentially varying depending on the values taken by other variables \cite{Pearl_2000}. In contrast, linear SCMs do not encounter such issues.
In the following, we give the definition of direct effects in the case of linear dynamic SCMs.

\begin{definition}[Direct Effect, \cite{Pearl_2000}]
In a linear dynamic SCM, the \emph{direct effect} of $X_{t-\gamma_{xy}}$ on $Y_t$ is fully specified by the structural coefficient\footnote{In Sewall Wright's terminology  the structural coefficient is called path coefficient \citep{Wright_1920,Wright_1921}.}  $\alpha_{X_{t-\gamma_{xy}}, Y_t}$.
\end{definition}

In the following, we explicitly state further assumptions we make in this paper.

\begin{assumption}[No hidden Confounding, \cite{Spirtes_2000,Pearl_2000}]
    \label{ass:CausalSufficiency}
    The noise variables in the linear dynamic SCM are assumed to be jointly independent (\ie, $\forall X, Y\in\mathcal{V}~\forall t',t\in\mathbb{Z},~X_{t'} \neq Y_{t} \implies \xi_{X_{t'}}\indep\xi_{Y_{t}}$).
\end{assumption}

\begin{assumption}[Stationarity]
    \label{ass:ConsistencyThroughoutTime}
    The causal mechanisms of the system considered do not change over time and therefore, $\forall X\neq Y\in \mathcal{V},~\forall t-\gamma \leq t\in \mathbb{Z},~\alpha_{X_{t-\gamma},Y_t}=\alpha_{X_{t-\gamma+1},Y_{t+1}},~\text{and}~\forall t-\gamma < t\in \mathbb{Z},~\alpha_{Y_{t-\gamma},Y_t}=\alpha_{Y_{t-\gamma+1},Y_{t+1}}$.
    There exists a maximum lag $\gamma_{max}$ of a dynamic SCM defined as $\gamma_{max} := max\{\gamma\in\mathbb{N}|\exists X,Y \in \mathcal{V},~\alpha_{X_{t-\gamma},Y_t} \neq 0\}$.
\end{assumption}

\begin{figure*}[t!]
\centering
   \begin{subfigure}{.19\textwidth}
    \centering
	\begin{tikzpicture}[{black, circle, draw, inner sep=0}]
	\tikzset{nodes={draw,rounded corners},minimum height=0.6cm,minimum width=0.6cm}	
	\tikzset{anomalous/.append style={fill=easyorange}}
	\tikzset{rc/.append style={fill=easyorange}}
	
	\node[fill=red!30] (X) at (0,0) {$X$} ;
	\node[fill=blue!30] (Y) at (2,0) {$Y$};
	\node (W) at (2,1.5) {$W$};
	\node (U) at (0,1.5) {$U$};
	\node (Z) at (1,0.75) {$Z$};
	
	\draw[->,>=latex, line width=2pt] (X) -- (Y);
        \draw [->,>=latex,] (U) --  (X);
         \begin{scope}[transform canvas={xshift=-.15em}]
          \draw [->,>=latex,] (Y) --  (W);
        \end{scope}
         \begin{scope}[transform canvas={xshift=.15em}]
          \draw [<-,>=latex,] (Y) --  (W);
        \end{scope}
        \begin{scope}[transform canvas={xshift=-.1em, yshift=-.1em}]
          \draw [->,>=latex,] (U) --  (Z);
        \end{scope}
         \begin{scope}[transform canvas={xshift=.1em, yshift=.1em}]
          \draw [<-,>=latex,] (U) --  (Z);
        \end{scope}
         \begin{scope}[transform canvas={xshift=-.1em, yshift=.1em}]
          \draw [->,>=latex,] (Z) --  (W);
        \end{scope}
         \begin{scope}[transform canvas={xshift=.1em, yshift=-.1em}]
          \draw [<-,>=latex,] (Z) --  (W);
        \end{scope}

        \draw[->,>=latex] (X) to [out=195,in=150, looseness=2] (X);
         \draw[->,>=latex] (Y) to [out=-15,in=30, looseness=2] (Y);
	\draw[->,>=latex] (Z) to [out=-30,in=15, looseness=2] (Z);
	\draw[->,>=latex] (U) to [out=180,in=135, looseness=2] (U);
	\draw[->,>=latex] (W) to [out=0,in=45, looseness=2] (W);
	\end{tikzpicture}
    \caption{\centering An SCG and the direct effect $\alpha_{X_{t-1},Y_t}$.}
    \label{fig:cond1:SCG1}
    \end{subfigure}
    \hfill 
        \begin{subfigure}{.39\textwidth}
    \centering
		\begin{tikzpicture}[{black, circle, draw, inner sep=0}]
		  \tikzset{nodes={draw,rounded corners},minimum height=0.7cm,minimum width=0.7cm}
		  \tikzset{latent/.append style={fill=gray!60}}
		  
            \node (X) at (2,4) {$X_t$};
            \node (U) at (2,3) {$U_t$};
            \node (W) at (2,1) {$W_t$};
            \node (V) at (2,2) {$Z_t$};
            \node[fill=blue!30] (Y) at (2,0) {$Y_t$};
            \node[fill=red!30] (X-1) at (0,4) {$X_{t-1}$};
            \node (U-1) at (0,3) {$U_{t-1}$};
            \node (W-1) at (0,1) {$W_{t-1}$};
            \node (V-1) at (0,2) {$Z_{t-1}$};
            \node (Y-1) at (0,0) {$Y_{t-1}$};
            \node (X-2) at (-2,4) {$X_{t-2}$};
            \node (U-2) at (-2,3) {$U_{t-2}$};
            \node (W-2) at (-2,1) {$W_{t-2}$};
            \node (V-2) at (-2,2) {$Z_{t-2}$};
            \node (Y-2) at (-2,0) {$Y_{t-2}$};
    
            \draw[->,>=latex, line width=2pt] (X-1) -- (Y);
            \draw[->,>=latex] (X-2) -- (Y-1);

            \draw[->,>=latex] (X) to [out=-30,in=30, looseness=0.8] (Y);
            \draw[->,>=latex] (X-1) to [out=-30,in=30, looseness=0.8] (Y-1);
            \draw[->,>=latex] (X-2) to [out=-30,in=30, looseness=0.8] (Y-2);

            \draw[->,>=latex] (U-1) -- (U);
            \draw[->,>=latex] (U-2) -- (U-1);
            \draw[->,>=latex] (V-1) -- (V);
            \draw[->,>=latex] (V-2) -- (V-1);
            \draw[->,>=latex] (W-1) -- (W);
            \draw[->,>=latex] (W-2) -- (W-1);
            \draw[->,>=latex] (Y-1) -- (Y);
            \draw[->,>=latex] (Y-2) -- (Y-1);
            \draw[->,>=latex] (X-1) -- (X);
            \draw[->,>=latex] (X-2) -- (X-1);

            \draw[->,>=latex] (Y-1) -- (W);
            \draw[->,>=latex] (Y-2) -- (W-1);
            \draw[->,>=latex] (W-1) -- (Y);
            \draw[->,>=latex] (W-2) -- (Y-1);
            \draw[->,>=latex] (U-1) -- (X);
            \draw[->,>=latex] (U-2) -- (X-1);

            \draw[->,>=latex] (V-1) -- (U);
            \draw[->,>=latex] (V-2) -- (U-1);
            \draw[->,>=latex] (W-1) -- (V);
            \draw[->,>=latex] (W-2) -- (V-1);

            \draw[->,>=latex] (U-1) -- (V);
            \draw[->,>=latex] (U-2) -- (V-1);
            \draw[->,>=latex] (V-1) -- (W);
            \draw[->,>=latex] (V-2) -- (W-1);

		
    		\coordinate[left of=X-2] (d1);
    		\draw [dashed,>=latex] (X-2) to[left] (d1);
    		\coordinate[left of=Y-2] (d1);
    		\draw [dashed,>=latex] (Y-2) to[left] (d1);		
    		\coordinate[left of=U-2] (d1);
    		\draw [dashed,>=latex] (U-2) to[left] (d1);		
    		\coordinate[left of=V-2] (d1);
    		\draw [dashed,>=latex] (V-2) to[left] (d1);		
    		\coordinate[left of=W-2] (d1);
    		\draw [dashed,>=latex] (W-2) to[left] (d1);
    		
    		\coordinate[right of=X] (d1);
    		\draw [dashed,>=latex] (X) to[right] (d1);
    		\coordinate[right of=Y] (d1);
    		\draw [dashed,>=latex] (Y) to[right] (d1);
    		\coordinate[right of=U] (d1);
    		\draw [dashed,>=latex] (U) to[right] (d1);
    		\coordinate[right of=V] (d1);
    		\draw [dashed,>=latex] (V) to[right] (d1);
    		\coordinate[right of=W] (d1);
    		\draw [dashed,>=latex] (W) to[right] (d1);
		\end{tikzpicture}
 	\caption{\centering A compatible FTCG.}
        \label{fig:cond1:FTCG1-1}
    \end{subfigure}
\hfill 
    \begin{subfigure}{.39\textwidth}
    \centering
		\begin{tikzpicture}[{black, circle, draw, inner sep=0}]
		  \tikzset{nodes={draw,rounded corners},minimum height=0.7cm,minimum width=0.7cm}
		  \tikzset{latent/.append style={fill=gray!60}}
		  
            \node (X) at (2,4) {$X_t$};
            \node (U) at (2,3) {$U_t$};
            \node (W) at (2,1) {$W_t$};
            \node (V) at (2,2) {$Z_t$};
            \node[fill=blue!30] (Y) at (2,0) {$Y_t$};
            \node[fill=red!30] (X-1) at (0,4) {$X_{t-1}$};
            \node (U-1) at (0,3) {$U_{t-1}$};
            \node (W-1) at (0,1) {$W_{t-1}$};
            \node (V-1) at (0,2) {$Z_{t-1}$};
            \node (Y-1) at (0,0) {$Y_{t-1}$};
            \node (X-2) at (-2,4) {$X_{t-2}$};
            \node (U-2) at (-2,3) {$U_{t-2}$};
            \node (W-2) at (-2,1) {$W_{t-2}$};
            \node (V-2) at (-2,2) {$Z_{t-2}$};
            \node (Y-2) at (-2,0) {$Y_{t-2}$};
    
            \draw[->,>=latex, line width=2pt] (X-1) -- (Y);
            \draw[->,>=latex] (X-2) -- (Y-1);


            \draw[->,>=latex] (U-1) -- (U);
            \draw[->,>=latex] (U-2) -- (U-1);
            \draw[->,>=latex] (V-1) -- (V);
            \draw[->,>=latex] (V-2) -- (V-1);
            \draw[->,>=latex] (W-1) -- (W);
            \draw[->,>=latex] (W-2) -- (W-1);
            \draw[->,>=latex] (Y-1) -- (Y);
            \draw[->,>=latex] (Y-2) -- (Y-1);
            \draw[->,>=latex] (X-1) -- (X);
            \draw[->,>=latex] (X-2) -- (X-1);

            \draw[->,>=latex] (Y-1) -- (W);
            \draw[->,>=latex] (Y-2) -- (W-1);
            \draw[->,>=latex] (W-1) -- (Y);
            \draw[->,>=latex] (W-2) -- (Y-1);
            \draw[->,>=latex] (U-1) -- (X);
            \draw[->,>=latex] (U-2) -- (X-1);

            \draw[->,>=latex] (V-1) -- (U);
            \draw[->,>=latex] (V-2) -- (U-1);
            \draw[->,>=latex] (W-1) -- (V);
            \draw[->,>=latex] (W-2) -- (V-1);

            \draw[->,>=latex] (U-1) -- (V);
            \draw[->,>=latex] (U-2) -- (V-1);
            \draw[->,>=latex] (V-1) -- (W);
            \draw[->,>=latex] (V-2) -- (W-1);

            \draw[->,>=latex] (U) to [out=45,in=-45, looseness=1] (X);
            \draw[->,>=latex] (U-1) to [out=45,in=-45, looseness=1] (X-1);
            \draw[->,>=latex] (U-2) to [out=45,in=-45, looseness=1] (X-2);
            \draw[<-,>=latex] (U) to [out=-45,in=45, looseness=1] (V);
            \draw[<-,>=latex] (U-1) to [out=-45,in=45, looseness=1] (V-1);
            \draw[<-,>=latex] (U-2) to [out=-45,in=45, looseness=1] (V-2);
            \draw[<-,>=latex] (V) to [out=-45,in=45, looseness=1] (W);
            \draw[<-,>=latex] (V-1) to [out=-45,in=45, looseness=1] (W-1);
            \draw[<-,>=latex] (V-2) to [out=-45,in=45, looseness=1] (W-2);
            \draw[<-,>=latex] (W) to [out=-45,in=45, looseness=1] (Y);
            \draw[<-,>=latex] (W-1) to [out=-45,in=45, looseness=1] (Y-1);
            \draw[<-,>=latex] (W-2) to [out=-45,in=45, looseness=1] (Y-2);
		
    		\coordinate[left of=X-2] (d1);
    		\draw [dashed,>=latex] (X-2) to[left] (d1);
    		\coordinate[left of=Y-2] (d1);
    		\draw [dashed,>=latex] (Y-2) to[left] (d1);		
    		\coordinate[left of=U-2] (d1);
    		\draw [dashed,>=latex] (U-2) to[left] (d1);		
    		\coordinate[left of=V-2] (d1);
    		\draw [dashed,>=latex] (V-2) to[left] (d1);		
    		\coordinate[left of=W-2] (d1);
    		\draw [dashed,>=latex] (W-2) to[left] (d1);
    		
    		\coordinate[right of=X] (d1);
    		\draw [dashed,>=latex] (X) to[right] (d1);
    		\coordinate[right of=Y] (d1);
    		\draw [dashed,>=latex] (Y) to[right] (d1);
    		\coordinate[right of=U] (d1);
    		\draw [dashed,>=latex] (U) to[right] (d1);
    		\coordinate[right of=V] (d1);
    		\draw [dashed,>=latex] (V) to[right] (d1);
    		\coordinate[right of=W] (d1);
    		\draw [dashed,>=latex] (W) to[right] (d1);
		\end{tikzpicture}
 	\caption{\centering Another compatible FTCG.}
        \label{fig:cond1:FTCG1-2}
    \end{subfigure}

   \begin{subfigure}{.19\textwidth}
    \centering
	\begin{tikzpicture}[{black, circle, draw, inner sep=0}]
	\tikzset{nodes={draw,rounded corners},minimum height=0.6cm,minimum width=0.6cm}	
	\tikzset{anomalous/.append style={fill=easyorange}}
	\tikzset{rc/.append style={fill=easyorange}}
	
	\node[fill=red!30] (X) at (0,0) {$X$} ;
	\node[fill=blue!30] (Y) at (2,0) {$Y$};
	\node (W) at (2,1.5) {$W$};
	\node (U) at (0,1.5) {$U$};
	\node (Z) at (1,0.75) {$Z$};
	
	\draw[->,>=latex, line width=2pt] (X) -- (Y);
        \draw [->,>=latex,] (U) --  (X);
         \begin{scope}[transform canvas={xshift=-.15em}]
          \draw [->,>=latex,] (Y) --  (W);
        \end{scope}
         \begin{scope}[transform canvas={xshift=.15em}]
          \draw [<-,>=latex,] (Y) --  (W);
        \end{scope}
        \begin{scope}[transform canvas={xshift=-.1em, yshift=-.1em}]
          \draw [->,>=latex,] (U) --  (Z);
        \end{scope}
         \begin{scope}[transform canvas={xshift=.1em, yshift=.1em}]
          \draw [<-,>=latex,] (U) --  (Z);
        \end{scope}
         \begin{scope}[transform canvas={xshift=-.1em, yshift=.1em}]
          \draw [->,>=latex,] (Z) --  (W);
        \end{scope}
         \begin{scope}[transform canvas={xshift=.1em, yshift=-.1em}]
          \draw [<-,>=latex,] (Z) --  (W);
        \end{scope}

        \draw[->,>=latex] (X) to [out=195,in=150, looseness=2] (X);
         \draw[->,>=latex] (Y) to [out=-15,in=30, looseness=2] (Y);
	\draw[->,>=latex] (Z) to [out=-30,in=15, looseness=2] (Z);
	\draw[->,>=latex] (U) to [out=180,in=135, looseness=2] (U);
	\draw[->,>=latex] (W) to [out=0,in=45, looseness=2] (W);
	\end{tikzpicture}
    \caption{\centering An SCG and the direct effect $\alpha_{X_{t},Y_t}$.}
    \label{fig:cond1:SCG2}
    \end{subfigure}
    \hfill 
        \begin{subfigure}{.39\textwidth}
    \centering
		\begin{tikzpicture}[{black, circle, draw, inner sep=0}]
		  \tikzset{nodes={draw,rounded corners},minimum height=0.7cm,minimum width=0.7cm}
		  \tikzset{latent/.append style={fill=gray!60}}
		  
            \node[fill=red!30] (X) at (2,4) {$X_t$};
            \node (U) at (2,3) {$U_t$};
            \node (W) at (2,1) {$W_t$};
            \node (V) at (2,2) {$Z_t$};
            \node[fill=blue!30] (Y) at (2,0) {$Y_t$};
            \node (X-1) at (0,4) {$X_{t-1}$};
            \node (U-1) at (0,3) {$U_{t-1}$};
            \node (W-1) at (0,1) {$W_{t-1}$};
            \node (V-1) at (0,2) {$Z_{t-1}$};
            \node (Y-1) at (0,0) {$Y_{t-1}$};
            \node (X-2) at (-2,4) {$X_{t-2}$};
            \node (U-2) at (-2,3) {$U_{t-2}$};
            \node (W-2) at (-2,1) {$W_{t-2}$};
            \node (V-2) at (-2,2) {$Z_{t-2}$};
            \node (Y-2) at (-2,0) {$Y_{t-2}$};
    
            \draw[->,>=latex] (X-1) -- (Y);
            \draw[->,>=latex] (X-2) -- (Y-1);


            \draw[->,>=latex] (U-1) -- (U);
            \draw[->,>=latex] (U-2) -- (U-1);
            \draw[->,>=latex] (V-1) -- (V);
            \draw[->,>=latex] (V-2) -- (V-1);
            \draw[->,>=latex] (W-1) -- (W);
            \draw[->,>=latex] (W-2) -- (W-1);
            \draw[->,>=latex] (Y-1) -- (Y);
            \draw[->,>=latex] (Y-2) -- (Y-1);
            \draw[->,>=latex] (X-1) -- (X);
            \draw[->,>=latex] (X-2) -- (X-1);

            \draw[->,>=latex] (Y-1) -- (W);
            \draw[->,>=latex] (Y-2) -- (W-1);
            \draw[->,>=latex] (W-1) -- (Y);
            \draw[->,>=latex] (W-2) -- (Y-1);
            \draw[->,>=latex] (U-1) -- (X);
            \draw[->,>=latex] (U-2) -- (X-1);

            \draw[->,>=latex] (V-1) -- (U);
            \draw[->,>=latex] (V-2) -- (U-1);
            \draw[->,>=latex] (W-1) -- (V);
            \draw[->,>=latex] (W-2) -- (V-1);

            \draw[->,>=latex] (U-1) -- (V);
            \draw[->,>=latex] (U-2) -- (V-1);
            \draw[->,>=latex] (V-1) -- (W);
            \draw[->,>=latex] (V-2) -- (W-1);

            \draw[->,>=latex] (U) to [out=45,in=-45, looseness=1] (X);
            \draw[->,>=latex] (U-1) to [out=45,in=-45, looseness=1] (X-1);
            \draw[->,>=latex] (U-2) to [out=45,in=-45, looseness=1] (X-2);
            \draw[<-,>=latex] (U) to [out=-45,in=45, looseness=1] (V);
            \draw[<-,>=latex] (U-1) to [out=-45,in=45, looseness=1] (V-1);
            \draw[<-,>=latex] (U-2) to [out=-45,in=45, looseness=1] (V-2);
            \draw[<-,>=latex] (V) to [out=-45,in=45, looseness=1] (W);
            \draw[<-,>=latex] (V-1) to [out=-45,in=45, looseness=1] (W-1);
            \draw[<-,>=latex] (V-2) to [out=-45,in=45, looseness=1] (W-2);
            \draw[<-,>=latex] (W) to [out=-45,in=45, looseness=1] (Y);
            \draw[<-,>=latex] (W-1) to [out=-45,in=45, looseness=1] (Y-1);
            \draw[<-,>=latex] (W-2) to [out=-45,in=45, looseness=1] (Y-2);
		
    		\coordinate[left of=X-2] (d1);
    		\draw [dashed,>=latex] (X-2) to[left] (d1);
    		\coordinate[left of=Y-2] (d1);
    		\draw [dashed,>=latex] (Y-2) to[left] (d1);		
    		\coordinate[left of=U-2] (d1);
    		\draw [dashed,>=latex] (U-2) to[left] (d1);		
    		\coordinate[left of=V-2] (d1);
    		\draw [dashed,>=latex] (V-2) to[left] (d1);		
    		\coordinate[left of=W-2] (d1);
    		\draw [dashed,>=latex] (W-2) to[left] (d1);
    		
    		\coordinate[right of=X] (d1);
    		\draw [dashed,>=latex] (X) to[right] (d1);
    		\coordinate[right of=Y] (d1);
    		\draw [dashed,>=latex] (Y) to[right] (d1);
    		\coordinate[right of=U] (d1);
    		\draw [dashed,>=latex] (U) to[right] (d1);
    		\coordinate[right of=V] (d1);
    		\draw [dashed,>=latex] (V) to[right] (d1);
    		\coordinate[right of=W] (d1);
    		\draw [dashed,>=latex] (W) to[right] (d1);
		\end{tikzpicture}
 	\caption{\centering A compatible FTCG.}
        \label{fig:cond1:FTCG2-1}
    \end{subfigure}
\hfill 
    \begin{subfigure}{.39\textwidth}
    \centering
		\begin{tikzpicture}[{black, circle, draw, inner sep=0}]
		  \tikzset{nodes={draw,rounded corners},minimum height=0.7cm,minimum width=0.7cm}
		  \tikzset{latent/.append style={fill=gray!60}}
		  
            \node[fill=red!30] (X) at (2,4) {$X_t$};
            \node (U) at (2,3) {$U_t$};
            \node (W) at (2,1) {$W_t$};
            \node (V) at (2,2) {$Z_t$};
            \node[fill=blue!30] (Y) at (2,0) {$Y_t$};
            \node (X-1) at (0,4) {$X_{t-1}$};
            \node (U-1) at (0,3) {$U_{t-1}$};
            \node (W-1) at (0,1) {$W_{t-1}$};
            \node (V-1) at (0,2) {$Z_{t-1}$};
            \node (Y-1) at (0,0) {$Y_{t-1}$};
            \node (X-2) at (-2,4) {$X_{t-2}$};
            \node (U-2) at (-2,3) {$U_{t-2}$};
            \node (W-2) at (-2,1) {$W_{t-2}$};
            \node (V-2) at (-2,2) {$Z_{t-2}$};
            \node (Y-2) at (-2,0) {$Y_{t-2}$};
    

            \draw[->,>=latex, line width=2pt] (X) to [out=-30,in=30, looseness=0.8] (Y);
            \draw[->,>=latex] (X-1) to [out=-30,in=30, looseness=0.8] (Y-1);
            \draw[->,>=latex] (X-2) to [out=-30,in=30, looseness=0.8] (Y-2);

            \draw[->,>=latex] (U-1) -- (U);
            \draw[->,>=latex] (U-2) -- (U-1);
            \draw[->,>=latex] (V-1) -- (V);
            \draw[->,>=latex] (V-2) -- (V-1);
            \draw[->,>=latex] (W-1) -- (W);
            \draw[->,>=latex] (W-2) -- (W-1);
            \draw[->,>=latex] (Y-1) -- (Y);
            \draw[->,>=latex] (Y-2) -- (Y-1);
            \draw[->,>=latex] (X-1) -- (X);
            \draw[->,>=latex] (X-2) -- (X-1);

            \draw[->,>=latex] (Y-1) -- (W);
            \draw[->,>=latex] (Y-2) -- (W-1);
            \draw[->,>=latex] (W-1) -- (Y);
            \draw[->,>=latex] (W-2) -- (Y-1);
            \draw[->,>=latex] (U-1) -- (X);
            \draw[->,>=latex] (U-2) -- (X-1);

            \draw[->,>=latex] (V-1) -- (U);
            \draw[->,>=latex] (V-2) -- (U-1);
            \draw[->,>=latex] (W-1) -- (V);
            \draw[->,>=latex] (W-2) -- (V-1);

            \draw[->,>=latex] (U-1) -- (V);
            \draw[->,>=latex] (U-2) -- (V-1);
            \draw[->,>=latex] (V-1) -- (W);
            \draw[->,>=latex] (V-2) -- (W-1);

            \draw[<-,>=latex] (U) to [out=-45,in=45, looseness=1] (V);
            \draw[<-,>=latex] (U-1) to [out=-45,in=45, looseness=1] (V-1);
            \draw[<-,>=latex] (U-2) to [out=-45,in=45, looseness=1] (V-2);
            \draw[<-,>=latex] (V) to [out=-45,in=45, looseness=1] (W);
            \draw[<-,>=latex] (V-1) to [out=-45,in=45, looseness=1] (W-1);
            \draw[<-,>=latex] (V-2) to [out=-45,in=45, looseness=1] (W-2);
            \draw[<-,>=latex] (W) to [out=-45,in=45, looseness=1] (Y);
            \draw[<-,>=latex] (W-1) to [out=-45,in=45, looseness=1] (Y-1);
            \draw[<-,>=latex] (W-2) to [out=-45,in=45, looseness=1] (Y-2);
		
    		\coordinate[left of=X-2] (d1);
    		\draw [dashed,>=latex] (X-2) to[left] (d1);
    		\coordinate[left of=Y-2] (d1);
    		\draw [dashed,>=latex] (Y-2) to[left] (d1);		
    		\coordinate[left of=U-2] (d1);
    		\draw [dashed,>=latex] (U-2) to[left] (d1);		
    		\coordinate[left of=V-2] (d1);
    		\draw [dashed,>=latex] (V-2) to[left] (d1);		
    		\coordinate[left of=W-2] (d1);
    		\draw [dashed,>=latex] (W-2) to[left] (d1);
    		
    		\coordinate[right of=X] (d1);
    		\draw [dashed,>=latex] (X) to[right] (d1);
    		\coordinate[right of=Y] (d1);
    		\draw [dashed,>=latex] (Y) to[right] (d1);
    		\coordinate[right of=U] (d1);
    		\draw [dashed,>=latex] (U) to[right] (d1);
    		\coordinate[right of=V] (d1);
    		\draw [dashed,>=latex] (V) to[right] (d1);
    		\coordinate[right of=W] (d1);
    		\draw [dashed,>=latex] (W) to[right] (d1);
		\end{tikzpicture}
 	\caption{\centering Another compatible FTCG.}
        \label{fig:cond1:FTCG2-2}
    \end{subfigure}

     \caption{An example of an SCG in (a and b) with four of its compatible FTCGs in (b), (c), (e) and (f).  The pair of red and blue vertices in the SCG and in the FTCGs represents the direct effect of interest, \ie, $\alpha_{X_t,Y_t}$ and $\alpha_{X_{t-1},Y_t}$.  
     In this example, $\alpha_{X_{t-1},Y_t}$ is non-identifiable given the SCG because for this direct effect, in the FTCG in (b), $X_t$ should be in every valid adjustment set but $X_t$ should not be in any valid adjustment set in the FTCG in (c). Similarly, $\alpha_{X_t,Y_t}$ is non-identifiable given the SCG because for this direct effect (equal to zero), in the FTCG in (e), at least one vertex in $\{U_t, Z_t, W_t\}$ should be in every valid adjustment set but none of the vertices in $\{U_t, Z_t, W_t\}$ should be in any valid adjustment set in the FTCG in (f).
     }
    \label{fig:cond1}
\end{figure*}
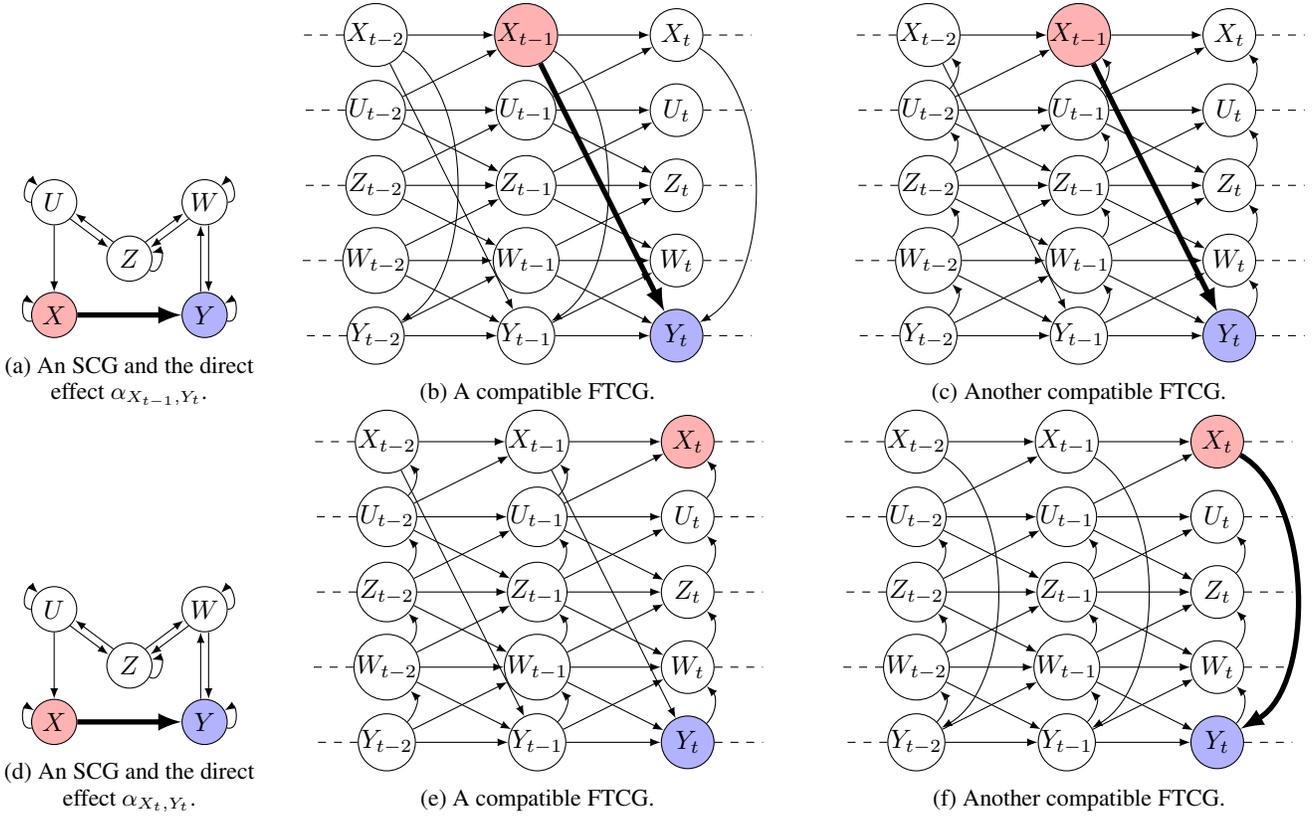%

\begin{definition}[Full-Time Causal Graph]
    \label{def:FTCG}
    Considering a finite set of times series $\mathcal{V}$ and a dynamic SCM, one can define the \emph{full-time causal graph (FTCG)} $\mathcal{G}_{f} = (\mathcal{V}_{f}, \mathcal{E}_{f})$ associated to the dynamic SCM in the following way:
    \begin{equation*}
        \begin{aligned}
            &\mathcal{V}_{f} &:=& \{Y_t &|& \forall Y \in \mathcal{V},~\forall t \in \mathbb{Z} \},&\\
            &\mathcal{E}_{f} &:=& \{X_{t-\gamma}\rightarrow Y_{t} &|& \forall X_{t-\gamma},Y_{t} \in \mathcal{V}_{f}& \\
            & & & & &\st \alpha_{X_{t-\gamma},Y_t} \neq 0\}.&
        \end{aligned}
    \end{equation*}
    In addition, for every vertex $Y_t \in \mathcal{V}_f$ we note:
\begin{itemize}
     \item $Par(Y_t,\mathcal{G}_f) = \{U_{t'} \in \mathcal{V}_f |U_{t'} \rightarrow Y_t$  \text{ in } $\mathcal{E}_f\}$,
     \item $Anc(Y_t,\mathcal{G}_f) = \bigcup_{n\in \mathbb{N}}P_n$ \text{ where } $P_0 = \{Y_t\}$\text{ and } $P_{k+1} = \bigcup_{U_{t'} \in P_k} Par(U_{t'},\mathcal{G}_f)$,\text{ and}
     \item $Desc(Y_t,\mathcal{G}_f) = \bigcup_{n\in \mathbb{N}}C_n$ \text{ where } $C_0 = \{Y_t\}$ \text{ and } $C_{k+1} = \bigcup_{U_{t'} \in C_k} \{W_{t''} \in \mathcal{V}_f |U_{t'} \rightarrow W_{t''}$ \text{ in } $\mathcal{E}_f\}$.
\end{itemize}
\end{definition}

\begin{assumption}[Acyclicity of the FTCG]
    \label{ass:AcyclicFTCG}
    Every FTCG is acyclic.
\end{assumption}

The FTCG is the most natural way to represent a dynamic  SCM but it is unpractical as it is infinite.
Of course, given Assumption~\ref{ass:ConsistencyThroughoutTime}, it is possible to represent an FTCG in a finite graph, but even in this case sometimes, it is difficult to construct this type of graphs from prior knowledge due to the uncertainty regarding temporal lags.
Furthermore, causal discovery methods are not always efficient~\citep{Ait_Bachir_2023} due to the strong assumptions they require that are not always satisfied in real applications, so constructing this type of graphs from data is not always a valid option.
Therefore, it is much more reliable to construct an abstraction of this type of graphs which does not contain temporal information.
This abstraction is usually referred to as a summary causal graph~\citep{Assaad_2022_a}.

\begin{definition}[Summary Causal Graph]
	\label{def:SCG}
    Considering a finite set of times series $\mathcal{V}$ and an FTCG $\mathcal{G}_{f} = (\mathcal{V}_{f}, \mathcal{E}_{f})$, one can define the \emph{summary causal graph (SCG)} $\mathcal{G}_{s} = (\mathcal{V}_{s}, \mathcal{E}_{s})$ compatible with the FTCG in the following way:
    \begin{equation*}
        \begin{aligned}
            &\mathcal{V}_{s} &:=& \mathcal{V},&\\
            &\mathcal{E}_{s} &:=& \{X\rightarrow Y &|& \forall X,Y \in \mathcal{V},&\\
            & & & & &~\exists t'\leq t\in \mathbb{Z} \st X_{t'}\rightarrow Y_{t}\in\mathcal{E}_{f}\}.&
        \end{aligned}
    \end{equation*}
        In addition, for every vertex $Y \in \mathcal{V}_s$ we note:
\begin{itemize}
    \item $Par(Y,\mathcal{G}_s) = \{U \in \mathcal{V}_s |U \rightarrow Y \text{ or }U\rightleftarrows Y$ \text{ in } $\mathcal{E}_s \}$,
    \item $Anc(Y,\mathcal{G}_s) = \bigcup_{n\in \mathbb{N}}P_n$ \text{ where } $P_0 = \{Y\}$ \text{ and } $P_{k+1} = \bigcup_{U \in P_k} Par(U,\mathcal{G}_s)$, \text{ and}
    \item $Desc(Y,\mathcal{G}_s) = \bigcup_{n\in \mathbb{N}}C_n$ \text{ where } $C_0 = \{Y\}$ \text{ and } $C_{k+1} = \bigcup_{U \in C_k} \{U \in \mathcal{V}_s |U \rightarrow W$ \text{ or } $U\rightleftarrows W \text{ in } \mathcal{E}_s\}$.
\end{itemize}
    Notice that an SCG may have cycles and in particular two arrows in opposite directions, i.e., if in the FTCG we have $X_{t'}\rightarrow Y_{t}$ and $Y_{t''}\rightarrow X_{t}$ then in the SCG we have $X\rightleftarrows Y$.
\end{definition}
	
The abstraction of SCGs entails that, even though there is exactly one SCG compatible with a given FTCG, there are in general several FTCGs compatible with a given SCG.
For example, the FTCGs in Figure~\ref{fig:cond1:FTCG1-1} and \ref{fig:cond1:FTCG1-2}
correspond to the same SCG given in Figure~\ref{fig:cond1:SCG1}.

Now that we have defined direct effects and SCGs, 
we can formally define graphical identifiability.

\begin{definition}[Graphical Identifiability of a Direct Effect from an SCG]
    The direct effect of a time instant $X_{t-\gamma_{xy}}$ on another time instant $Y_t$, i.e., $\alpha_{X_{t-\gamma_{xy}},Y_t}$, in a linear dynamic SCM is said to be \emph{identifiable} from an SCG if the quantity $\alpha_{X_{t-\gamma_{xy}},Y_t}$ can be computed uniquely from the observed distribution without any further assumption on the distribution and without knowing the FTCG.
\end{definition}
Note that given the true FTCG, computing the direct effect uniquely from the observed distribution without any further assumption on the distribution usually consists of removing all confounding bias and non-direct effects by adjusting on a set that do not create any selection bias. Removing confounding bias usually necessitates adjusting on some suitable ancestors of the cause or the effect, removing non-direct effects usually consists of adjusting on some suitable ancestors of the effect that are not ancestors of the cause, and creating selection bias consists of adjusting on descendants of the effect. In the linear setting, given such an adjustment set $\mathcal{Z}_f$, the direct effect $\alpha_{X_{t-\gamma}, Y_t}$ can be estimated using the partial linear regression coefficient $r_{X_{t-\gamma} Y_t . \mathcal{Z}_f}$  where $r_{X_{t-\gamma} Y_t . \mathcal{Z}_f}$ represents the correlation between $X_{t-\gamma}$ and $Y_t$ after $\mathcal{Z}_f$ is "partialled out" \cite{Pearl_2000}. The difficulty of this paper lies in finding such an adjustment set $\mathcal{Z}_f$ using the SCG and without knowing the true FTCG. This consists in finding an adjustment set that is valid for every FTCG compatible with the SCG.

\begin{figure*}[t!]
\centering
   \begin{subfigure}{.19\textwidth}
    \centering
	\begin{tikzpicture}[{black, circle, draw, inner sep=0}]
	\tikzset{nodes={draw,rounded corners},minimum height=0.6cm,minimum width=0.6cm}	
	\tikzset{anomalous/.append style={fill=easyorange}}
	\tikzset{rc/.append style={fill=easyorange}}
	
	\node[fill=red!30] (X) at (0,0) {$X$} ;
	\node[fill=blue!30] (Y) at (2,0) {$Y$};
	\node (W) at (2,1.5) {$W$};
	\node (U) at (0,1.5) {$U$};
	\node (Z) at (1,0.75) {$Z$};
	
	\draw[->,>=latex, line width=2pt] (X) -- (Y);
        \draw [->,>=latex,] (U) --  (X);
         \begin{scope}[transform canvas={xshift=-.15em}]
          \draw [->,>=latex,] (Y) --  (W);
        \end{scope}
         \begin{scope}[transform canvas={xshift=.15em}]
          \draw [<-,>=latex,] (Y) --  (W);
        \end{scope}
        \begin{scope}[transform canvas={xshift=-.1em, yshift=-.1em}]
          \draw [->,>=latex,] (U) --  (Z);
        \end{scope}
         \begin{scope}[transform canvas={xshift=.1em, yshift=.1em}]
          \draw [<-,>=latex,] (U) --  (Z);
        \end{scope}
         \begin{scope}[transform canvas={xshift=-.1em, yshift=.1em}]
          \draw [->,>=latex,] (Z) --  (W);
        \end{scope}
         \begin{scope}[transform canvas={xshift=.1em, yshift=-.1em}]
          \draw [<-,>=latex,] (Z) --  (W);
        \end{scope}
        
	\draw[->,>=latex] (Y) to [out=-15,in=30, looseness=2] (Y);
	\draw[->,>=latex] (Z) to [out=-30,in=15, looseness=2] (Z);
	\draw[->,>=latex] (U) to [out=180,in=135, looseness=2] (U);
	\draw[->,>=latex] (W) to [out=0,in=45, looseness=2] (W);
	\end{tikzpicture}
    \caption{\centering An SCG  and the direct effect $\alpha_{X_{t},Y_t}$.}
    \label{fig:cond2a:SCG}
    \end{subfigure}
    \hfill 
        \begin{subfigure}{.39\textwidth}
    \centering
		\begin{tikzpicture}[{black, circle, draw, inner sep=0}]
		  \tikzset{nodes={draw,rounded corners},minimum height=0.7cm,minimum width=0.7cm}
		  \tikzset{latent/.append style={fill=gray!60}}
		  
            \node[fill=red!30] (X) at (2,4) {$X_t$};
            \node (U) at (2,3) {$U_t$};
            \node (W) at (2,1) {$W_t$};
            \node (V) at (2,2) {$Z_t$};
            \node[fill=blue!30] (Y) at (2,0) {$Y_t$};
            \node (X-1) at (0,4) {$X_{t-1}$};
            \node (U-1) at (0,3) {$U_{t-1}$};
            \node (W-1) at (0,1) {$W_{t-1}$};
            \node (V-1) at (0,2) {$Z_{t-1}$};
            \node (Y-1) at (0,0) {$Y_{t-1}$};
            \node (X-2) at (-2,4) {$X_{t-2}$};
            \node (U-2) at (-2,3) {$U_{t-2}$};
            \node (W-2) at (-2,1) {$W_{t-2}$};
            \node (V-2) at (-2,2) {$Z_{t-2}$};
            \node (Y-2) at (-2,0) {$Y_{t-2}$};

            \draw[->,>=latex] (X-1) -- (Y);
            \draw[->,>=latex] (X-2) -- (Y-1);

            \draw[->,>=latex, line width=2pt] (X) to [out=-30,in=30, looseness=0.8] (Y);
            \draw[->,>=latex] (X-1) to [out=-30,in=30, looseness=0.8] (Y-1);
            \draw[->,>=latex] (X-2) to [out=-30,in=30, looseness=0.8] (Y-2);

            \draw[->,>=latex] (U-1) -- (U);
            \draw[->,>=latex] (U-2) -- (U-1);
            \draw[->,>=latex] (V-1) -- (V);
            \draw[->,>=latex] (V-2) -- (V-1);
            \draw[->,>=latex] (W-1) -- (W);
            \draw[->,>=latex] (W-2) -- (W-1);
            \draw[->,>=latex] (Y-1) -- (Y);
            \draw[->,>=latex] (Y-2) -- (Y-1);

            \draw[->,>=latex] (Y-1) -- (W);
            \draw[->,>=latex] (Y-2) -- (W-1);
            \draw[->,>=latex] (W-1) -- (Y);
            \draw[->,>=latex] (W-2) -- (Y-1);
            \draw[->,>=latex] (U-1) -- (X);
            \draw[->,>=latex] (U-2) -- (X-1);

            \draw[->,>=latex] (V-1) -- (U);
            \draw[->,>=latex] (V-2) -- (U-1);
            \draw[->,>=latex] (W-1) -- (V);
            \draw[->,>=latex] (W-2) -- (V-1);

            \draw[->,>=latex] (U-1) -- (V);
            \draw[->,>=latex] (U-2) -- (V-1);
            \draw[->,>=latex] (V-1) -- (W);
            \draw[->,>=latex] (V-2) -- (W-1);

            \draw[->,>=latex] (U) to [out=45,in=-45, looseness=1] (X);
            \draw[->,>=latex] (U-1) to [out=45,in=-45, looseness=1] (X-1);
            \draw[->,>=latex] (U-2) to [out=45,in=-45, looseness=1] (X-2);
            \draw[->,>=latex] (U) to [out=-45,in=45, looseness=1] (V);
            \draw[->,>=latex] (U-1) to [out=-45,in=45, looseness=1] (V-1);
            \draw[->,>=latex] (U-2) to [out=-45,in=45, looseness=1] (V-2);
            \draw[->,>=latex] (V) to [out=-45,in=45, looseness=1] (W);
            \draw[->,>=latex] (V-1) to [out=-45,in=45, looseness=1] (W-1);
            \draw[->,>=latex] (V-2) to [out=-45,in=45, looseness=1] (W-2);
            \draw[->,>=latex] (W) to [out=-45,in=45, looseness=1] (Y);
            \draw[->,>=latex] (W-1) to [out=-45,in=45, looseness=1] (Y-1);
            \draw[->,>=latex] (W-2) to [out=-45,in=45, looseness=1] (Y-2);
		
    		\coordinate[left of=X-2] (d1);
    		\draw [dashed,>=latex] (X-2) to[left] (d1);
    		\coordinate[left of=Y-2] (d1);
    		\draw [dashed,>=latex] (Y-2) to[left] (d1);		
    		\coordinate[left of=U-2] (d1);
    		\draw [dashed,>=latex] (U-2) to[left] (d1);		
    		\coordinate[left of=V-2] (d1);
    		\draw [dashed,>=latex] (V-2) to[left] (d1);		
    		\coordinate[left of=W-2] (d1);
    		\draw [dashed,>=latex] (W-2) to[left] (d1);
    		
    		\coordinate[right of=X] (d1);
    		\draw [dashed,>=latex] (X) to[right] (d1);
    		\coordinate[right of=Y] (d1);
    		\draw [dashed,>=latex] (Y) to[right] (d1);
    		\coordinate[right of=U] (d1);
    		\draw [dashed,>=latex] (U) to[right] (d1);
    		\coordinate[right of=V] (d1);
    		\draw [dashed,>=latex] (V) to[right] (d1);
    		\coordinate[right of=W] (d1);
    		\draw [dashed,>=latex] (W) to[right] (d1);
		\end{tikzpicture}
 	\caption{\centering A compatible FTCG.}
        \label{fig:cond2a:FTCG1}
    \end{subfigure}
\hfill 
    \begin{subfigure}{.39\textwidth}
    \centering
		\begin{tikzpicture}[{black, circle, draw, inner sep=0}]
		  \tikzset{nodes={draw,rounded corners},minimum height=0.7cm,minimum width=0.7cm}
		  \tikzset{latent/.append style={fill=gray!60}}
		  
            \node[fill=red!30] (X) at (2,4) {$X_t$};
            \node (U) at (2,3) {$U_t$};
            \node (W) at (2,1) {$W_t$};
            \node (V) at (2,2) {$Z_t$};
            \node[fill=blue!30] (Y) at (2,0) {$Y_t$};
            \node (X-1) at (0,4) {$X_{t-1}$};
            \node (U-1) at (0,3) {$U_{t-1}$};
            \node (W-1) at (0,1) {$W_{t-1}$};
            \node (V-1) at (0,2) {$Z_{t-1}$};
            \node (Y-1) at (0,0) {$Y_{t-1}$};
            \node (X-2) at (-2,4) {$X_{t-2}$};
            \node (U-2) at (-2,3) {$U_{t-2}$};
            \node (W-2) at (-2,1) {$W_{t-2}$};
            \node (V-2) at (-2,2) {$Z_{t-2}$};
            \node (Y-2) at (-2,0) {$Y_{t-2}$};

            \draw[->,>=latex] (X-1) -- (Y);
            \draw[->,>=latex] (X-2) -- (Y-1);

            \draw[->,>=latex, line width=2pt] (X) to [out=-30,in=30, looseness=0.8] (Y);
            \draw[->,>=latex] (X-1) to [out=-30,in=30, looseness=0.8] (Y-1);
            \draw[->,>=latex] (X-2) to [out=-30,in=30, looseness=0.8] (Y-2);

            \draw[->,>=latex] (U-1) -- (U);
            \draw[->,>=latex] (U-2) -- (U-1);
            \draw[->,>=latex] (V-1) -- (V);
            \draw[->,>=latex] (V-2) -- (V-1);
            \draw[->,>=latex] (W-1) -- (W);
            \draw[->,>=latex] (W-2) -- (W-1);
            \draw[->,>=latex] (Y-1) -- (Y);
            \draw[->,>=latex] (Y-2) -- (Y-1);

            \draw[->,>=latex] (Y-1) -- (W);
            \draw[->,>=latex] (Y-2) -- (W-1);
            \draw[->,>=latex] (W-1) -- (Y);
            \draw[->,>=latex] (W-2) -- (Y-1);
            \draw[->,>=latex] (U-1) -- (X);
            \draw[->,>=latex] (U-2) -- (X-1);

            \draw[->,>=latex] (V-1) -- (U);
            \draw[->,>=latex] (V-2) -- (U-1);
            \draw[->,>=latex] (W-1) -- (V);
            \draw[->,>=latex] (W-2) -- (V-1);

            \draw[->,>=latex] (U-1) -- (V);
            \draw[->,>=latex] (U-2) -- (V-1);
            \draw[->,>=latex] (V-1) -- (W);
            \draw[->,>=latex] (V-2) -- (W-1);

            \draw[<-,>=latex] (U) to [out=-45,in=45, looseness=1] (V);
             \draw[<-,>=latex] (U-1) to [out=-45,in=45, looseness=1] (V-1);
            \draw[<-,>=latex] (U-2) to [out=-45,in=45, looseness=1] (V-2);
            \draw[<-,>=latex] (V) to [out=-45,in=45, looseness=1] (W);
            \draw[<-,>=latex] (V-1) to [out=-45,in=45, looseness=1] (W-1);
            \draw[<-,>=latex] (V-2) to [out=-45,in=45, looseness=1] (W-2);
            \draw[<-,>=latex] (W) to [out=-45,in=45, looseness=1] (Y);
            \draw[<-,>=latex] (W-1) to [out=-45,in=45, looseness=1] (Y-1);
            \draw[<-,>=latex] (W-2) to [out=-45,in=45, looseness=1] (Y-2);
		
    		\coordinate[left of=X-2] (d1);
    		\draw [dashed,>=latex] (X-2) to[left] (d1);
    		\coordinate[left of=Y-2] (d1);
    		\draw [dashed,>=latex] (Y-2) to[left] (d1);		
    		\coordinate[left of=U-2] (d1);
    		\draw [dashed,>=latex] (U-2) to[left] (d1);		
    		\coordinate[left of=V-2] (d1);
    		\draw [dashed,>=latex] (V-2) to[left] (d1);		
    		\coordinate[left of=W-2] (d1);
    		\draw [dashed,>=latex] (W-2) to[left] (d1);
    		
    		\coordinate[right of=X] (d1);
    		\draw [dashed,>=latex] (X) to[right] (d1);
    		\coordinate[right of=Y] (d1);
    		\draw [dashed,>=latex] (Y) to[right] (d1);
    		\coordinate[right of=U] (d1);
    		\draw [dashed,>=latex] (U) to[right] (d1);
    		\coordinate[right of=V] (d1);
    		\draw [dashed,>=latex] (V) to[right] (d1);
    		\coordinate[right of=W] (d1);
    		\draw [dashed,>=latex] (W) to[right] (d1);
		\end{tikzpicture}
 	\caption{\centering Another compatible FTCG.}
        \label{fig:cond2a:FTCG2}
    \end{subfigure}

     \caption{An example of an SCG in (a) with two of its compatible FTCGs in (b) and (c).  The pair of red and blue vertices in the SCG and in the FTCGs represents the direct effect of interest, \ie, $\alpha_{X_t,Y_t}$.  
     In this example, $\alpha_{X_t,Y_t}$ is non-identifiable given the SCG because for this direct effect at least one vertex in $\{U_t, Z_t, W_t\}$ should be in every valid adjustment set for the first FTCG but none of the vertices in $\{U_t, Z_t, W_t\}$ should be in any valid adjustment set for the second FTCG.}
   \label{fig:cond2a}
\end{figure*}
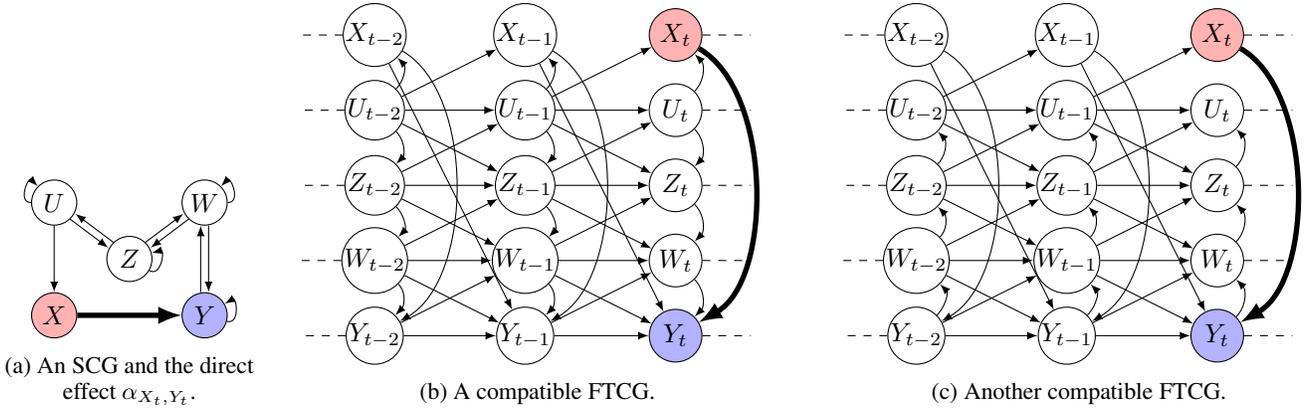%

In the following, we recall several preliminaries related to FTCGs and SCGs.

\begin{definition}[Paths in FTCGs]
    \label{def:Path_FTCG}
    A \emph{path} between two vertices $X_{t'}$ to $Y_{t}$ is an ordered sequence of vertices denoted as $\pi_f=\langle V^1_{t^1},\dots,V^n_{t^n} \rangle$ such that $V^1_{t^1}=X_{t'}$, $V^n_{t^n}=Y_{t}$ and $\forall 1\leq i < n$, $V^i_{t^i}$ and $V^{i+1}_{t^{i+1}}$ are adjacent (\ie, $V^i_{t^i}\rightarrow V^{i+1}_{t^{i+1}}\text{ or }V^i_{t^{i}}\leftarrow V^{i+1}_{t^{i+1}}$) and $\forall 1\leq i<j \leq n,~V^i_{t^i}\neq V^j_{t^j}$.
    In cases where orientations are crucial in the context under discussion, in the path $\pi_f$, we will replace commas with orientations.
    In this paper, a path $\pi_f$ from $X_{t'}$ to $Y_{t}$ is said to be \emph{non-direct} if $\pi_f \neq \langle X_{t'}\rightarrow Y_{t} \rangle$.
\end{definition}

\begin{definition}[Paths in SCGs]
    \label{def:Paths_SCG}
    A \emph{path} between two vertices $X$ to $Y$ is an ordered sequence of vertices denoted as $\pi_s=\langle V^1,\dots,V^n \rangle$ such that $V^1=X$, $V^n=Y$ and $\forall 1\leq i < n$, $V^i$ and $V^{i+1}$ are adjacent (\ie, $V^i\rightarrow V^{i+1}\text{ or }V^i\leftarrow V^{i+1}\text{ or }V^i\rightleftarrows V^{i+1}$) and $\forall 1\leq i<j \leq n,~V^i\neq V^j$.
    In cases where orientations are crucial in the context under discussion, in the path $\pi_s$, we will replace commas with orientations.
    In this paper, a path $\pi_s$ from $X$ to $Y$ is said to be \emph{non-direct} if $\pi_s \neq \langle X \rightarrow Y \rangle$.
\end{definition}

A path is said to be \emph{directed} if it only contains $\rightarrow$ and $\leftrightarrows$.

\begin{definition}[Cycles in SCGs]
    \label{def:Cycle}
    In an SCG, a \emph{cycle} is an ordered sequence of vertices $\pi_s=\langle V^1,\dots,V^n \rangle$ such that:
    \begin{itemize}
        \item $V^1=V^n$, 
        \item $\forall 1\leq i < n,~V^i\rightarrow V^{i+1} \text{ or } V^i\rightleftarrows V^{i+1}$, and
        \item $\forall 1\leq i < j \leq n,~V^i=V^j \implies i=1 \text{ and }j=n$.
    \end{itemize}
    The set of cycles with endpoints $Y \in \mathcal{V}_s$ in an SCG $\mathcal{G}_s$ is written $Cycles(Y,\mathcal{G}_s)$.
\end{definition}


An important graphical notion used in causal reasoning is the notion of blocked path~\citep{Pearl_1998} for which the classical definition  was introduced for directed acyclic graphs and thus can be directly used for FTCGs under Assumption~\ref{ass:AcyclicFTCG}.
Note that an adjustment set that removes all confounding bias and non-direct effects and that do not create any selection bias between $X_{t-\gamma_{xy}}$ and $Y_t$ consists of non descendants of $Y_t$ different from $X_{t-\gamma_{xy}}$ that block all non-direct paths between $X_{t-\gamma_{xy}}$ and $Y_t$.
In the following, we firstly give the definition of  blocked path in FTCGs and then we give a similar notion for SCGs.

\begin{definition}[Blocked Path in FTCGs]
    \label{def:BlockedPathInFTCG}
    In an FTCG $\mathcal{G}_f=(\mathcal{V}_f,\mathcal{E}_f)$, a path $\pi_f=\langle V^1_{t^1},\dots,V^n_{t^n} \rangle$ is said to be \emph{blocked} by a set of vertices $\mathcal{Z}_f\subseteq\mathcal{V}_f$ if:
    \begin{enumerate}
        \item \label{item:BlockedByEndpointsInFTCG} $V^1_{t^1}\in\mathcal{Z}_f$ or $V^n_{t^n}\in\mathcal{Z}_f$, or
        \item \label{item:ActivelyBlockedInFTCG} $\exists 1<i<n \st V^{i-1}_{t^{i-1}}\leftarrow V^i_{t^{i}}$ or $V^i_{t^i}\rightarrow V^{i+1}_{t^{i+1}}$ and $V^i_{t^{i}}\in\mathcal{Z}_f$, or
        \item \label{item:PassivelyBlockedInFTCG} $\exists 1<i<n \st \rightarrow V^i_{t^{i}}\leftarrow$ and $Desc(V^i_{t^{i}}, \mathcal{G}_f)\cap\mathcal{Z}_f=\emptyset$.
    \end{enumerate}
    A path which is not blocked is said to be \emph{active}.
    When the set $\mathcal{Z}_f$ is not specified, it is implicit that we consider $\mathcal{Z}_f=\emptyset$.
\end{definition}

The classical definition of blocked path is usually used in directed acyclic graph and since the SCG compatible with an FTCG can be cyclic, one needs to adapt it.
\cite{Spirtes_1993} explains that under the linearity assumption the notion of blocked path is readily extended.
Moreover, \cite{Mooij_2017} introduced a more recent and general (non-parametric and allow for hidden confounding) adaptation called $\sigma$-blocked path.
Here we will focus on the definition used in \cite{Spirtes_1993} since we assume linearity but we adapt it such that in SCGs extremity vertices of a path do not necessarily block the  path.

\begin{definition}[Blocked Path in SCGs]
    \label{def:BlockedPathInSCG}
    In an SCG $\mathcal{G}_s=(\mathcal{V}_s,\mathcal{E}_s)$, a path $\pi_s=\langle V^1,\dots,V^n \rangle$ is said to be \emph{blocked} by a set of vertices $\mathcal{Z}_s\subseteq\mathcal{V}_s$ if:
    \begin{enumerate}
        \item \label{item:ActivelyBlockedInSCG} $\exists 1 < i < n \st V^{i-1}\leftarrow V^i$ or $V^i\rightarrow V^{i+1}$ and $V^i\in\mathcal{Z}_s$, or
        \item \label{item:PassivelyBlockedInSCG} $\exists 1 < i \leq j < n \st \langle V^{i-1}\rightarrow V^i\rightleftarrows\dots\rightleftarrows V^{j} \leftarrow V^{j+1} \rangle \text{ and }Desc(V^i, \mathcal{G}_s) \cap \mathcal{Z}_s=\emptyset$.
    \end{enumerate}
    A path which is not blocked is said to be \emph{active}.
    When the set $\mathcal{Z}_s$ is not specified, it is implicit that we consider $\mathcal{Z}_s=\emptyset$.
\end{definition}

Condition~\ref{item:ActivelyBlockedInSCG} in Definition~\ref{def:BlockedPathInSCG} is a direct adaptation of condition~\ref{item:ActivelyBlockedInFTCG} in Definition~\ref{def:BlockedPathInFTCG}.
Condition~\ref{item:PassivelyBlockedInSCG} of Definition~\ref{def:BlockedPathInSCG} is explained by the fact that for a path $\pi_s=\langle V^1,\dots,V^n\rangle$ in an SCG $\mathcal{G}_s=(\mathcal{V}_s,\mathcal{E}_s)$ and a set of vertices $\mathcal{Z}_s\subseteq\mathcal{V}_s$, if $\exists 1 < i \leq j < n \st$
\begin{align*}
&\langle V^{i-1}\rightarrow V^i\rightleftarrows\dots\rightleftarrows V^{j} \leftarrow V^{j+1}\rangle \text{ and }\\
&Desc(V^i, \mathcal{G}_s) \cap \mathcal{Z}_s=\emptyset
\end{align*}
then 
$\forall \pi_f=\langle V^1_{t^1},\dots,V^n_{t^n}\rangle, \exists 1 < i \leq k \leq j < n \st$
\begin{align*}
\rightarrow V^k_{t^{k}}\leftarrow
\text{ and } Desc(V^k, \mathcal{G}_s)\cap\mathcal{Z}_s=\emptyset
\end{align*}
$\text{therefore }Desc(V^k_{t^{k}}, \mathcal{G}_f)\cap\mathcal{Z}_f=\emptyset$ \text{where } $\mathcal{Z}_f\subseteq \{V_{t'}|V\in \mathcal{Z}_s,~t'\in\mathbb{Z}\}$.
Notice that there is no adaptation of condition~\ref{item:BlockedByEndpointsInFTCG} as having $V^1\in\mathcal{Z}_s$ or $V^n\in\mathcal{Z}_s$ does not mean that instants of interests $V^1_{t^1}$ or $V^n_{t^n}$ which are endpoints of compatible paths of interests are in $\mathcal{Z}_f\subseteq \{V_{t'}|V\in \mathcal{Z}_s,~t'\in\mathbb{Z}\}$.
Moreover, in this paper we are interested in the direct effect of $V^1_{t^1}=X_{t-\gamma_{xy}}$ on $V^n_{t^n}=Y_t$ and therefore we cannot adjust on them and thus we will always have $V^1_{t^1},V^n_{t^n}\notin \mathcal{Z}_f$.

Note that a set $\mathcal{Z}_s$ that blocks all paths between two vertices $X$ and $Y$ in an SCG does not necessary have a compatible \textit{finite} set $\mathcal{Z}_f\subseteq \{V_{t'}|V\in \mathcal{Z}_s,~t'\in\mathbb{Z}\}$ that block all paths between two vertices $X_{t-\gamma_{xy}}$ and $Y_t$ in every FTCG compatible with the given SCG. For example, in Figure~\ref{fig:cond1:SCG2}, $U$ blocks all paths between $X$ and $Z$ but in  Figure~\ref{fig:cond1:FTCG2-1}, $\not\exists i\in\mathbb{N}$ such that $\{U_{t-i}, \cdots, U_t\}$ blocks all paths between $X_t$ and $Z_t$ because there will always be an active path between them passing by  $U_{t-i-1}$.

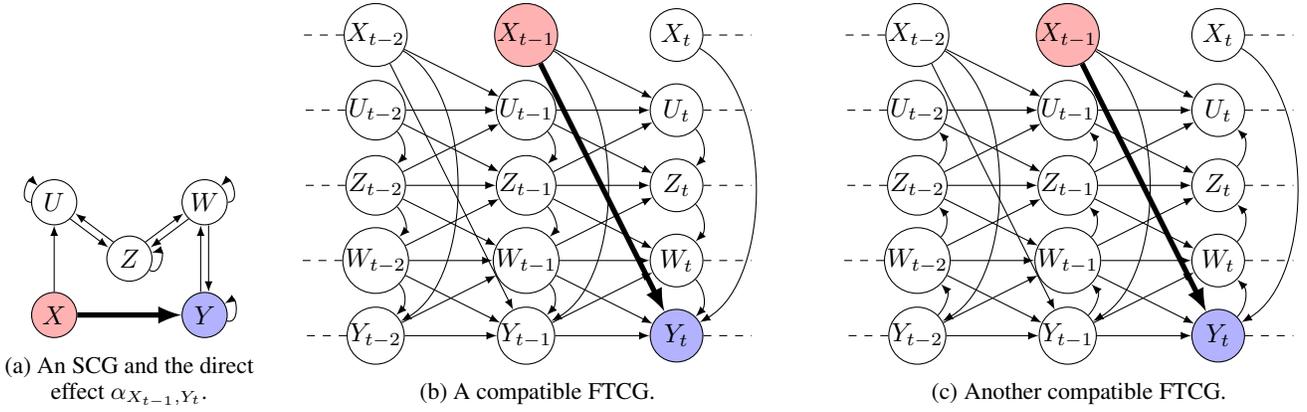
\begin{figure*}[t!]
\centering
   \begin{subfigure}{.19\textwidth}
    \centering
	\begin{tikzpicture}[{black, circle, draw, inner sep=0}]
	\tikzset{nodes={draw,rounded corners},minimum height=0.6cm,minimum width=0.6cm}	
	\tikzset{anomalous/.append style={fill=easyorange}}
	\tikzset{rc/.append style={fill=easyorange}}
	
	\node[fill=red!30] (X) at (0,0) {$X$} ;
	\node[fill=blue!30] (Y) at (2,0) {$Y$};
	\node (W) at (2,1.5) {$W$};
	\node (U) at (0,1.5) {$U$};
	\node (Z) at (1,0.75) {$Z$};
	
	\draw[->,>=latex, line width=2pt] (X) -- (Y);
        \draw [->,>=latex,] (X) --  (U);
         \begin{scope}[transform canvas={xshift=-.15em}]
          \draw [->,>=latex,] (Y) --  (W);
        \end{scope}
         \begin{scope}[transform canvas={xshift=.15em}]
          \draw [<-,>=latex,] (Y) --  (W);
        \end{scope}
        \begin{scope}[transform canvas={xshift=-.1em, yshift=-.1em}]
          \draw [->,>=latex,] (U) --  (Z);
        \end{scope}
         \begin{scope}[transform canvas={xshift=.1em, yshift=.1em}]
          \draw [<-,>=latex,] (U) --  (Z);
        \end{scope}
         \begin{scope}[transform canvas={xshift=-.1em, yshift=.1em}]
          \draw [->,>=latex,] (Z) --  (W);
        \end{scope}
         \begin{scope}[transform canvas={xshift=.1em, yshift=-.1em}]
          \draw [<-,>=latex,] (Z) --  (W);
        \end{scope}
        
	\draw[->,>=latex] (Y) to [out=-15,in=30, looseness=2] (Y);
	\draw[->,>=latex] (Z) to [out=-30,in=15, looseness=2] (Z);
	\draw[->,>=latex] (U) to [out=180,in=135, looseness=2] (U);
	\draw[->,>=latex] (W) to [out=0,in=45, looseness=2] (W);
	\end{tikzpicture}
    \caption{\centering An SCG  and the direct effect $\alpha_{X_{t-1},Y_t}$.}
    \label{fig:cond2b:SCG}
    \end{subfigure}
    \hfill 
        \begin{subfigure}{.39\textwidth}
    \centering
		\begin{tikzpicture}[{black, circle, draw, inner sep=0}]
		  \tikzset{nodes={draw,rounded corners},minimum height=0.7cm,minimum width=0.7cm}
		  \tikzset{latent/.append style={fill=gray!60}}
		  
            \node (X) at (2,4) {$X_t$};
            \node (U) at (2,3) {$U_t$};
            \node (W) at (2,1) {$W_t$};
            \node (V) at (2,2) {$Z_t$};
            \node[fill=blue!30] (Y) at (2,0) {$Y_t$};
            \node[fill=red!30] (X-1) at (0,4) {$X_{t-1}$};
            \node (U-1) at (0,3) {$U_{t-1}$};
            \node (W-1) at (0,1) {$W_{t-1}$};
            \node (V-1) at (0,2) {$Z_{t-1}$};
            \node (Y-1) at (0,0) {$Y_{t-1}$};
            \node (X-2) at (-2,4) {$X_{t-2}$};
            \node (U-2) at (-2,3) {$U_{t-2}$};
            \node (W-2) at (-2,1) {$W_{t-2}$};
            \node (V-2) at (-2,2) {$Z_{t-2}$};
            \node (Y-2) at (-2,0) {$Y_{t-2}$};

            \draw[->,>=latex, line width=2pt] (X-1) -- (Y);
            \draw[->,>=latex] (X-2) -- (Y-1);

            \draw[->,>=latex] (X) to [out=-30,in=30, looseness=0.8] (Y);
            \draw[->,>=latex] (X-1) to [out=-30,in=30, looseness=0.8] (Y-1);
            \draw[->,>=latex] (X-2) to [out=-30,in=30, looseness=0.8] (Y-2);

            \draw[->,>=latex] (U-1) -- (U);
            \draw[->,>=latex] (U-2) -- (U-1);
            \draw[->,>=latex] (V-1) -- (V);
            \draw[->,>=latex] (V-2) -- (V-1);
            \draw[->,>=latex] (W-1) -- (W);
            \draw[->,>=latex] (W-2) -- (W-1);
            \draw[->,>=latex] (Y-1) -- (Y);
            \draw[->,>=latex] (Y-2) -- (Y-1);

            \draw[->,>=latex] (Y-1) -- (W);
            \draw[->,>=latex] (Y-2) -- (W-1);
            \draw[->,>=latex] (W-1) -- (Y);
            \draw[->,>=latex] (W-2) -- (Y-1);
            \draw[->,>=latex] (X-1) -- (U);
            \draw[->,>=latex] (X-2) -- (U-1);

            \draw[->,>=latex] (V-1) -- (U);
            \draw[->,>=latex] (V-2) -- (U-1);
            \draw[->,>=latex] (W-1) -- (V);
            \draw[->,>=latex] (W-2) -- (V-1);

            \draw[->,>=latex] (U-1) -- (V);
            \draw[->,>=latex] (U-2) -- (V-1);
            \draw[->,>=latex] (V-1) -- (W);
            \draw[->,>=latex] (V-2) -- (W-1);

            \draw[->,>=latex] (U) to [out=-45,in=45, looseness=1] (V);
            \draw[->,>=latex] (U-1) to [out=-45,in=45, looseness=1] (V-1);
            \draw[->,>=latex] (U-2) to [out=-45,in=45, looseness=1] (V-2);
            \draw[->,>=latex] (V) to [out=-45,in=45, looseness=1] (W);
            \draw[->,>=latex] (V-1) to [out=-45,in=45, looseness=1] (W-1);
            \draw[->,>=latex] (V-2) to [out=-45,in=45, looseness=1] (W-2);
            \draw[->,>=latex] (W) to [out=-45,in=45, looseness=1] (Y);
            \draw[->,>=latex] (W-1) to [out=-45,in=45, looseness=1] (Y-1);
            \draw[->,>=latex] (W-2) to [out=-45,in=45, looseness=1] (Y-2);
		
    		\coordinate[left of=X-2] (d1);
    		\draw [dashed,>=latex] (X-2) to[left] (d1);
    		\coordinate[left of=Y-2] (d1);
    		\draw [dashed,>=latex] (Y-2) to[left] (d1);		
    		\coordinate[left of=U-2] (d1);
    		\draw [dashed,>=latex] (U-2) to[left] (d1);		
    		\coordinate[left of=V-2] (d1);
    		\draw [dashed,>=latex] (V-2) to[left] (d1);		
    		\coordinate[left of=W-2] (d1);
    		\draw [dashed,>=latex] (W-2) to[left] (d1);
    		
    		\coordinate[right of=X] (d1);
    		\draw [dashed,>=latex] (X) to[right] (d1);
    		\coordinate[right of=Y] (d1);
    		\draw [dashed,>=latex] (Y) to[right] (d1);
    		\coordinate[right of=U] (d1);
    		\draw [dashed,>=latex] (U) to[right] (d1);
    		\coordinate[right of=V] (d1);
    		\draw [dashed,>=latex] (V) to[right] (d1);
    		\coordinate[right of=W] (d1);
    		\draw [dashed,>=latex] (W) to[right] (d1);
		\end{tikzpicture}
 	\caption{\centering A compatible FTCG.}
        \label{fig:cond2b:FTCG1}
    \end{subfigure}
\hfill 
    \begin{subfigure}{.39\textwidth}
    \centering
		\begin{tikzpicture}[{black, circle, draw, inner sep=0}]
		  \tikzset{nodes={draw,rounded corners},minimum height=0.7cm,minimum width=0.7cm}
		  \tikzset{latent/.append style={fill=gray!60}}
		  
            \node (X) at (2,4) {$X_t$};
            \node (U) at (2,3) {$U_t$};
            \node (W) at (2,1) {$W_t$};
            \node (V) at (2,2) {$Z_t$};
            \node[fill=blue!30] (Y) at (2,0) {$Y_t$};
            \node[fill=red!30] (X-1) at (0,4) {$X_{t-1}$};
            \node (U-1) at (0,3) {$U_{t-1}$};
            \node (W-1) at (0,1) {$W_{t-1}$};
            \node (V-1) at (0,2) {$Z_{t-1}$};
            \node (Y-1) at (0,0) {$Y_{t-1}$};
            \node (X-2) at (-2,4) {$X_{t-2}$};
            \node (U-2) at (-2,3) {$U_{t-2}$};
            \node (W-2) at (-2,1) {$W_{t-2}$};
            \node (V-2) at (-2,2) {$Z_{t-2}$};
            \node (Y-2) at (-2,0) {$Y_{t-2}$};

            \draw[->,>=latex] (X) to [out=-30,in=30, looseness=0.8] (Y);
            \draw[->,>=latex] (X-1) to [out=-30,in=30, looseness=0.8] (Y-1);
            \draw[->,>=latex] (X-2) to [out=-30,in=30, looseness=0.8] (Y-2);
            \draw[->,>=latex, line width=2pt] (X-1) -- (Y);
            \draw[->,>=latex] (X-2) -- (Y-1);


            \draw[->,>=latex] (U-1) -- (U);
            \draw[->,>=latex] (U-2) -- (U-1);
            \draw[->,>=latex] (V-1) -- (V);
            \draw[->,>=latex] (V-2) -- (V-1);
            \draw[->,>=latex] (W-1) -- (W);
            \draw[->,>=latex] (W-2) -- (W-1);
            \draw[->,>=latex] (Y-1) -- (Y);
            \draw[->,>=latex] (Y-2) -- (Y-1);

            \draw[->,>=latex] (Y-1) -- (W);
            \draw[->,>=latex] (Y-2) -- (W-1);
            \draw[->,>=latex] (W-1) -- (Y);
            \draw[->,>=latex] (W-2) -- (Y-1);
            \draw[->,>=latex] (X-1) -- (U);
            \draw[->,>=latex] (X-2) -- (U-1);

            \draw[->,>=latex] (V-1) -- (U);
            \draw[->,>=latex] (V-2) -- (U-1);
            \draw[->,>=latex] (W-1) -- (V);
            \draw[->,>=latex] (W-2) -- (V-1);

            \draw[->,>=latex] (U-1) -- (V);
            \draw[->,>=latex] (U-2) -- (V-1);
            \draw[->,>=latex] (V-1) -- (W);
            \draw[->,>=latex] (V-2) -- (W-1);

            \draw[<-,>=latex] (U) to [out=-45,in=45, looseness=1] (V);
             \draw[<-,>=latex] (U-1) to [out=-45,in=45, looseness=1] (V-1);
            \draw[<-,>=latex] (U-2) to [out=-45,in=45, looseness=1] (V-2);
            \draw[<-,>=latex] (V) to [out=-45,in=45, looseness=1] (W);
            \draw[<-,>=latex] (V-1) to [out=-45,in=45, looseness=1] (W-1);
            \draw[<-,>=latex] (V-2) to [out=-45,in=45, looseness=1] (W-2);
            \draw[<-,>=latex] (W) to [out=-45,in=45, looseness=1] (Y);
            \draw[<-,>=latex] (W-1) to [out=-45,in=45, looseness=1] (Y-1);
            \draw[<-,>=latex] (W-2) to [out=-45,in=45, looseness=1] (Y-2);
		
    		\coordinate[left of=X-2] (d1);
    		\draw [dashed,>=latex] (X-2) to[left] (d1);
    		\coordinate[left of=Y-2] (d1);
    		\draw [dashed,>=latex] (Y-2) to[left] (d1);		
    		\coordinate[left of=U-2] (d1);
    		\draw [dashed,>=latex] (U-2) to[left] (d1);		
    		\coordinate[left of=V-2] (d1);
    		\draw [dashed,>=latex] (V-2) to[left] (d1);		
    		\coordinate[left of=W-2] (d1);
    		\draw [dashed,>=latex] (W-2) to[left] (d1);
    		
    		\coordinate[right of=X] (d1);
    		\draw [dashed,>=latex] (X) to[right] (d1);
    		\coordinate[right of=Y] (d1);
    		\draw [dashed,>=latex] (Y) to[right] (d1);
    		\coordinate[right of=U] (d1);
    		\draw [dashed,>=latex] (U) to[right] (d1);
    		\coordinate[right of=V] (d1);
    		\draw [dashed,>=latex] (V) to[right] (d1);
    		\coordinate[right of=W] (d1);
    		\draw [dashed,>=latex] (W) to[right] (d1);
		\end{tikzpicture}
 	\caption{\centering Another compatible FTCG.}
        \label{fig:cond2b:FTCG2}
    \end{subfigure}

     \caption{An example of an SCG in (a) with two of its compatible FTCGs in (b) and (c).  The pair of red and blue vertices in the SCG and in the FTCGs represents the direct effect of interest, \ie, $\alpha_{X_{t-1},Y_t}$.  
     In this example, $\alpha_{X_{t-1},Y_t}$ is non-identifiable given the SCG because for this direct effect at least one vertex in $\{U_t, Z_t, W_t\}$ should be in every valid adjustment set for the first FTCG but none of the vertices in $\{U_t, Z_t, W_t\}$ should be in any valid adjustment set for the second FTCG.}
   \label{fig:cond2b}
\end{figure*}%

\section{Complete Graphical Identifiability from SCGs}
\label{sec:identification}
In this section, we start by presenting the complete identifiability result followed by its proof as well as several toy examples to demonstrate it. 
Then we give a weaker but interesting result that is implied by the identifiability result.

\begin{restatable}{theorem}{mytheoremone}
\label{theorem:identifiability}
    Let $\mathcal{G}_s=(\mathcal{V}_s, \mathcal{E}_s)$ be an SCG that represents a linear dynamic SCM verifying Assumptions~\ref{ass:CausalSufficiency},\ref{ass:ConsistencyThroughoutTime},\ref{ass:AcyclicFTCG}, $\gamma_{max} \geq 0$ a maximum lag and $\alpha_{X_{t-\gamma_{xy}},Y_t}$ the direct effect of $X_{t-\gamma_{xy}}$ on $Y_t \st X, Y \in \mathcal{V}_s$ and $X \neq Y$. The direct effect $\alpha_{X_{t-\gamma_{xy}}, Y_t}$ is non-identifiable from $\mathcal{G}_s$ if and only if $X \in Par(Y,\mathcal{G}_s)$ and one of the following conditions holds:
    \begin{enumerate}
    \item \label{item:1} $X \in Desc(Y,\mathcal{G}_s)$ and $\exists C\in Cycles(X,\mathcal{G}_s)$ with $Y \notin C$, or
    \item \label{item:2} There exists an active non-direct path $\pi_s=\langle V^1,\dots,V^n \rangle$ from $X$ to $Y$ in $\mathcal{G}_s$ such that $\langle V^2,\dots,V^{n-1} \rangle \subseteq Desc(Y, \mathcal{G}_s)$, and one of the following conditions holds:
    \begin{enumerate}
        \item \label{item:2.1} $\gamma_{xy}=0$, or
        \item \label{item:2.2} $\gamma_{xy}>0$, $n\geq 3$ and $\nexists 1 \leq i < n,~V^i \leftarrow V^{i+1}$ (\ie, $\forall i,~V^i \rightarrow V^{i+1}\text{ or }V^i \rightleftarrows V^{i+1}$).
    \end{enumerate}
    \end{enumerate}
\end{restatable}


In the following, we provide an additional definition that will be utilized in proving Theorem~\ref{theorem:identifiability}, followed by the proof itself.
A more detailed proof of Theorem~\ref{theorem:identifiability} is also available in the supplementary materials.
    
\begin{definition}[A First Finite Adjustment Set]
    \label{def:huge-single-door_set}
    Consider an SCG $\mathcal{G}_s=(\mathcal{V}_s,\mathcal{E}_s)$, a maximal lag $\gamma_{max}$, two vertices $X$ and $Y$ with $X\in Par(Y,\mathcal{G}_s)$ and a lag $\gamma_{xy}$. $\mathcal{Z}_{f}^{\text{Def~\ref{def:huge-single-door_set}}}= \mathcal{A}_{\leq t} \cup \mathcal{D}_{<t}$ is an adjustment set relative to $(X_{t-\gamma_{xy}}, Y_t)$ such that:
\begin{align*}
        \mathcal{D}_{<t} &= \{V_{t'} &|& V \in Desc(Y, \mathcal{G}_s),& \\
        & & &t-\gamma_{max} \leq t' < t\} \backslash \{X_{t-\gamma_{xy}}\}& \text{ and }\\
        \mathcal{A}_{\leq t} &= \{V_{t'} &|& V \in \mathcal{V}_{s} \backslash Desc(Y, \mathcal{G}_s),&  \\
        & & &t-\gamma_{max} \leq t' \leq t\} \backslash \{X_{t-\gamma_{xy}}\}.&
\end{align*}
\end{definition}

\begin{proof}
    Firstly, let us prove the backward implication.

    Let $\gamma_{xy} \geq 0$ and $X\in Par(Y,\mathcal{G}_s)$. 
    \begin{itemize}
        
        \item Suppose Condition~\ref{item:1} holds.
        Let $C=\langle V^1,\dots,V^{n} \rangle$ be a cycle on $X$ with $Y\notin C$.
        If $\gamma_{xy} > 0$ and since $X\in Par(Y,\mathcal{G}_s)$, there exists a compatible FTCG $\mathcal{G}_f^1$ where there exists a directed path $\pi_f^1 = \langle V^1_{t-\gamma_{xy}}=X_{t-\gamma_{xy}},V^2_t,\dots,V^{n}_t={X_t},Y_t \rangle$. Obviously, each set of vertices $\mathcal{Z}_f^1\subset \mathcal{V}_f^1$ that blocks all non-direct paths between $X_{t-\gamma_{xy}}$ and $Y_t$ have to contain at least one vertex from $\langle V^2_t,\dots,V^{n-1}_t \rangle$.
        Similarly, since $X\in Desc(Y, \mathcal{G}_s)$,
        there exists a compatible FTCG $\mathcal{G}_f^2$ where $\{V^2_t,\dots,V^{n}_t={X_t}\} \subseteq Desc(Y_t, \mathcal{G}_f^2)$. Thus each set of vertices $\mathcal{Z}_f^2\subset \mathcal{V}_f^2$ that blocks all non-direct paths between $X_{t-\gamma_{xy}}$ and $Y_t$ should not contain any vertex from $\langle V^2_t,\dots,V^{n}_t \rangle$ in order to be a valid adjustment set.
        If $\gamma_{xy}=0$, then since $X\in Desc(Y,\mathcal{G}_s)$ there exists $\langle V^n,\dots,V^1 \rangle$ a directed path from $Y$ to $X$ and there exists $\pi_f = \langle V^1_{t},\dots,V^{n}_t \rangle$ in a compatible FTCG.
        If $n=2$, then $\pi_f=\langle V^1_t \leftarrow V^n_t \rangle$ cannot be blocked and if $n\geq3$, then every set $\mathcal{Z}_f$ that blocks this path contains a descendant of $Y_t$ in another compatible FTCG.
        
        \item Suppose Condition~\ref{item:2} holds.
        Let $\pi_s = \langle V^1,\dots,V^n \rangle$ as described.
        If $\pi_s = \langle X \rightleftarrows Y \rangle$ and $\gamma_{xy}=0$, then there exists a compatible FTCG $\mathcal{G}_f$ in which the path $\pi_f = \langle X_t \leftarrow Y_t \rangle$ exists and cannot be blocked.
        Else, $n \geq 3$ and $\pi_s$ is active so there exists a compatible FTCG $\mathcal{G}_f^1$ in which the path $\pi_f = \langle V^1_{t-\gamma_{xy}},V^2_t,\dots,V^{n-1}_t,V^n_t \rangle$ exists and is active.
        Notice that there exists another compatible FTCG $\mathcal{G}_f^2$ in which $\{V^2_t,\dots,V^{n-1}_t\} \subseteq Desc(Y_t,\mathcal{G}_f^2)$ and every set $\mathcal{Z}_f$ that blocks $\pi_f$ in $\mathcal{G}_f^1$ contains a vertex in $Desc(Y_t,\mathcal{G}_f^2)$.
    \end{itemize}

    Secondly, let us prove the forward implication.
    Let $\mathcal{G}_f$ be an FTCG, $\gamma_{xy}\geq 0$ a lag, $X_{t-\gamma_{xy}}$ and $Y_t$ two vertices and $\mathcal{G}_s$ the SCG compatible with $\mathcal{G}_f$.
    If $X\notin Par(Y,\mathcal{G}_s)$ then 
    there is no direct effect from $X_{t-\gamma_{xy}}$ to $Y_t$, \ie, $\alpha_{X_{t-\gamma_{xy}},Y_t}=0$.
    Suppose that $X\in Par(Y,\mathcal{G}_s)$ and that the conditions of Theorem~\ref{theorem:identifiability} are not verified and let us show that $\mathcal{Z}_f^{\text{Def~\ref{def:huge-single-door_set}}}$ is a valid adjustment set to estimate the direct effect $\alpha_{X_{t-\gamma_{xy}}, Y_t}$ by showing that $\mathcal{Z}_f^{\text{Def~\ref{def:huge-single-door_set}}}$ verifies:
    \begin{itemize}
        \item $\mathcal{Z}_f^{\text{Def~\ref{def:huge-single-door_set}}} \cap (Desc(Y_t,\mathcal{G}_f) \cup \{X_{t-\gamma_{xy}}\}) = \emptyset$, and
        \item $\mathcal{Z}_f^{\text{Def~\ref{def:huge-single-door_set}}}$ blocks every non-direct path from $X_{t-\gamma_{xy}}$ to $Y_t$ in every compatible FTCG $\mathcal{G}_f$ of maximal lag at most $\gamma_{max}$.
    \end{itemize}
    The first point, $\mathcal{Z}_f^{\text{Def~\ref{def:huge-single-door_set}}} \cap (Desc(Y_t,\mathcal{G}_f) \cup \{X_{t-\gamma_{xy}}\}) = \emptyset$ is immediate.
    In order to prove the second point, let $\pi_f = \langle V^1_{t^1},\dots,V^n_{t^n} \rangle$ be a non-direct path from $X_{t-\gamma_{xy}}$ to $Y_t$.
    Notice the following properties:
    \begin{itemize}
    \item If $\exists i_{max} = max\{1 < i < n | t^i <t\}$ then $t-\gamma_{xy} \leq t^{i_{max}} < t$ and $V^{i_{max}}_{t^{i_{max}}} \rightarrow V^{i_{max}+1}_{t^{i_{max}+1}}$ thus $\pi_f$ is blocked by $\mathcal{Z}_f^{\text{Def~\ref{def:huge-single-door_set}}}$ as $V^{i_{max}}_{t^{i_{max}}} \in \mathcal{Z}_f^{\text{Def~\ref{def:huge-single-door_set}}}$.
    \item Else, if $I=\{1 < i < n | t^i > t\} \neq \emptyset$ then $\exists i \in I$ such that $\rightarrow V^i_{t^i} \leftarrow$ in $\pi_f$ and $\pi_f$ is $\mathcal{Z}_f^{\text{Def~\ref{def:huge-single-door_set}}}$-blocked as $Desc(V^i_{t^i},\mathcal{G}_f) \cap \mathcal{Z}_f^{\text{Def~\ref{def:huge-single-door_set}}} = \emptyset$.
    \item Else, if $\exists i_{max} = max\{1 < i < n | V^i \notin Desc(Y,\mathcal{G}_s)\}$ then $t^{i_{max}} = t$ and $V^{i_{max}}_{t^{i_{max}}} \rightarrow V^{i_{max}+1}_{t^{i_{max}+1}}$ thus $\pi_f$ is blocked by $\mathcal{Z}_f^{\text{Def~\ref{def:huge-single-door_set}}}$ as $V^{i_{max}}_{t^{i_{max}}} \in \mathcal{Z}_f^{\text{Def~\ref{def:huge-single-door_set}}}$.
    \end{itemize}
    Therefore, if $\pi_f$ is $\mathcal{Z}_f^{\text{Def~\ref{def:huge-single-door_set}}}$-active then $\langle V^2,\dots,V^n \rangle \subseteq Desc(Y,\mathcal{G}_s)$ and either $\pi_f=\langle X_{t-\gamma_{xy}} \leftarrow Y_t \rangle$ which forces $\gamma_{xy} = 0$ or $n\geq 3$ and $t^2=\dots=t^{n-1}= t$.
    Notice that $\pi_f$ and $\pi_s = \langle V^1,\dots,V^n \rangle$ are active and non-direct.
    \begin{itemize}
        \item Suppose $\pi_f=\langle X_{t} \leftarrow Y_t \rangle$ ($\gamma_{xy}=0$). Given that $X\in Par(Y,\mathcal{G}_s)$, this means that $\langle X\rightleftarrows Y \rangle$ is a path in $\mathcal{G}_s$ and thus Condition~\ref{item:2.1} is verified.
    \item 
    Suppose $n\geq 3$ and $t^2=\dots=t^{n-1}= t$.
    We make the following \emph{observation}: \emph{if $\gamma_{xy}>0$, then $t -\gamma_{xy} = t^1 < t^2 = t$ so $V^1_{t^1} \rightarrow V^2_{t^2}$ and given that $\pi_f$ is active, $\forall 1 \leq i < n,~V^i_{t^i} \rightarrow V^{i+1}_{t^{i+1}}$ and $\nexists 1 \leq i < n,~V^i \leftarrow V^{i+1}$.}
    In addition, since $\pi_f$ is a path and $t^2=\dots=t^{n-1}= t$, either $\pi_s$ is a path or $\gamma_{xy}>0$ and $\exists 1<i<n \st V^i=X$ and $\forall 1\leq j_1<j_2 \leq n,~ V^{j_1}=V^{j_2}\implies j_1=1$ and $j_2=i$.
    
    \begin{itemize}
        \item 
    Suppose $\pi_s$ is a path. In the case $\gamma_{xy} = 0$, Condition~\ref{item:2.1} would be verified. In the case $\gamma_{xy}>0$, by the previous \emph{observation}, Condition~\ref{item:2.2} would be verified.
    
    \item
    Suppose $\gamma_{xy}>0$ and $\exists 1<i<n \st V^i=X$ and $\forall 1\leq j_1<j_2 \leq n,~ V^{j_1}=V^{j_2}\implies j_1=1$ and $j_2=i$. Notice that $X=V^i \in Desc(Y,\mathcal{G}_s)$ and that, using the previous \emph{observation}, $Y\notin \langle V^1,\dots,V^i \rangle\in Cycles(X,\mathcal{G}_s)$.
    Thus Condition~\ref{item:1} is verified.
    \end{itemize}
    \end{itemize}
    In conclusion, when the conditions of Theorem~\ref{theorem:identifiability} are not verified, there is no non-direct $\mathcal{Z}_f^{\text{Def~\ref{def:huge-single-door_set}}}$-active path $\pi_f$ between $X_{t-\gamma_{xy}}$ and $Y_t$.
\end{proof}

\begin{figure*}[ht!]
   \begin{subfigure}{\myscale\textwidth}
	\begin{tikzpicture}[{black, circle, draw, inner sep=0}]
	\tikzset{nodes={draw,rounded corners},minimum height=0.6cm,minimum width=0.6cm}	
	\tikzset{anomalous/.append style={fill=easyorange}}
	\tikzset{rc/.append style={fill=easyorange}}
	
	\node[fill=red!30] (X) at (0*\mysecondscale,0*\mysecondscale) {$X$} ;
	\node[fill=blue!30] (Y) at (2*\mysecondscale,0*\mysecondscale) {$Y$};
	\node (W) at (2*\mysecondscale,1.5*\mysecondscale) {$W$};
	\node (U) at (0*\mysecondscale,1.5*\mysecondscale) {$U$};
	\node (Z) at (1*\mysecondscale,0.75*\mysecondscale) {$Z$};
	
	\draw[->,>=latex, line width=2pt] (X) -- (Y);
        \draw [->,>=latex,] (U) --  (X);
        \begin{scope}[transform canvas={xshift=-.15em}]
          \draw [->,>=latex,] (Y) --  (W);
        \end{scope}
         \begin{scope}[transform canvas={xshift=.15em}]
          \draw [<-,>=latex,] (Y) --  (W);
        \end{scope}
         \begin{scope}[transform canvas={xshift=-.1em, yshift=.1em}]
          \draw [->,>=latex,] (Z) --  (W);
        \end{scope}
         \begin{scope}[transform canvas={xshift=.1em, yshift=-.1em}]
          \draw [<-,>=latex,] (Z) --  (W);
        \end{scope}	
        \draw[->,>=latex] (U) -- (Z);

	\draw[->,>=latex] (X) to [out=195,in=150, looseness=2] (X);
	\draw[->,>=latex] (Y) to [out=-15,in=30, looseness=2] (Y);
	\draw[->,>=latex] (Z) to [out=-30,in=15, looseness=2] (Z);
	\draw[->,>=latex] (U) to [out=180,in=135, looseness=2] (U);
	\draw[->,>=latex] (W) to [out=0,in=45, looseness=2] (W);
	
	\end{tikzpicture}
 	\caption{}
	\label{fig:identifiable_1}
  \end{subfigure}
    \hfill 
   \begin{subfigure}{\myscale\textwidth}
	\begin{tikzpicture}[{black, circle, draw, inner sep=0}]
	\tikzset{nodes={draw,rounded corners},minimum height=0.6cm,minimum width=0.6cm}	
	\tikzset{anomalous/.append style={fill=easyorange}}
	\tikzset{rc/.append style={fill=easyorange}}
	
	\node[fill=red!30] (X) at (0*\mysecondscale,0*\mysecondscale) {$X$} ;
	\node[fill=blue!30] (Y) at (2*\mysecondscale,0*\mysecondscale) {$Y$};
	\node (W) at (2*\mysecondscale,1.5*\mysecondscale) {$W$};
	\node (U) at (0*\mysecondscale,1.5*\mysecondscale) {$U$};
	\node (Z) at (1*\mysecondscale,0.75*\mysecondscale) {$Z$};

	\draw[->,>=latex, line width=2pt] (X) -- (Y);
        \draw [->,>=latex,] (U) --  (X);
        \begin{scope}[transform canvas={xshift=-.15em}]
          \draw [->,>=latex,] (Y) --  (W);
        \end{scope}
         \begin{scope}[transform canvas={xshift=.15em}]
          \draw [<-,>=latex,] (Y) --  (W);
        \end{scope}
         \begin{scope}[transform canvas={xshift=-.1em, yshift=.1em}]
          \draw [->,>=latex,] (Z) --  (W);
        \end{scope}
         \begin{scope}[transform canvas={xshift=.1em, yshift=-.1em}]
          \draw [<-,>=latex,] (Z) --  (W);
        \end{scope}	
        \draw[->,>=latex] (U) -- (Z);

	\draw[->,>=latex] (Y) to [out=-15,in=30, looseness=2] (Y);
	\draw[->,>=latex] (Z) to [out=-30,in=15, looseness=2] (Z);
	\draw[->,>=latex] (U) to [out=180,in=135, looseness=2] (U);
	\draw[->,>=latex] (W) to [out=0,in=45, looseness=2] (W);
	
	\end{tikzpicture}
 	\caption{}
	\label{fig:identifiable_2}
  \end{subfigure}
      \hfill 
   \begin{subfigure}{\myscale\textwidth}
	\begin{tikzpicture}[{black, circle, draw, inner sep=0}]
	\tikzset{nodes={draw,rounded corners},minimum height=0.6cm,minimum width=0.6cm}	
	\tikzset{anomalous/.append style={fill=easyorange}}
	\tikzset{rc/.append style={fill=easyorange}}
	
	\node[fill=red!30] (X) at (0*\mysecondscale,0*\mysecondscale) {$X$} ;
	\node[fill=blue!30] (Y) at (2*\mysecondscale,0*\mysecondscale) {$Y$};
	\node (W) at (2*\mysecondscale,1.5*\mysecondscale) {$W$};
	\node (U) at (0*\mysecondscale,1.5*\mysecondscale) {$U$};
	\node (Z) at (1*\mysecondscale,0.75*\mysecondscale) {$Z$};

	\draw[->,>=latex, line width=2pt] (X) -- (Y);
        \draw [<-,>=latex,] (U) --  (X);
        \begin{scope}[transform canvas={xshift=-.15em}]
          \draw [->,>=latex,] (Y) --  (W);
        \end{scope}
         \begin{scope}[transform canvas={xshift=.15em}]
          \draw [<-,>=latex,] (Y) --  (W);
        \end{scope}
         \begin{scope}[transform canvas={xshift=-.1em, yshift=.1em}]
          \draw [->,>=latex,] (Z) --  (W);
        \end{scope}
         \begin{scope}[transform canvas={xshift=.1em, yshift=-.1em}]
          \draw [<-,>=latex,] (Z) --  (W);
        \end{scope}	
        \draw[->,>=latex] (U) -- (Z);

	\draw[->,>=latex] (X) to [out=195,in=150, looseness=2] (X);
	\draw[->,>=latex] (Y) to [out=-15,in=30, looseness=2] (Y);
	\draw[->,>=latex] (Z) to [out=-30,in=15, looseness=2] (Z);
	\draw[->,>=latex] (U) to [out=180,in=135, looseness=2] (U);
	\draw[->,>=latex] (W) to [out=0,in=45, looseness=2] (W);
	
	\end{tikzpicture}
 	\caption{}
	\label{fig:identifiable_3}
  \end{subfigure}
    \hfill 
   \begin{subfigure}{\myscale\textwidth}
	\begin{tikzpicture}[{black, circle, draw, inner sep=0}]
	\tikzset{nodes={draw,rounded corners},minimum height=0.6cm,minimum width=0.6cm}	
	\tikzset{anomalous/.append style={fill=easyorange}}
	\tikzset{rc/.append style={fill=easyorange}}
	
	\node[fill=red!30] (X) at (0*\mysecondscale,0*\mysecondscale) {$X$} ;
	\node[fill=blue!30] (Y) at (2*\mysecondscale,0*\mysecondscale) {$Y$};
	\node (W) at (2*\mysecondscale,1.5*\mysecondscale) {$W$};
	\node (U) at (0*\mysecondscale,1.5*\mysecondscale) {$U$};
	\node (Z) at (1*\mysecondscale,0.75*\mysecondscale) {$Z$};

	\draw[->,>=latex, line width=2pt] (X) -- (Y);
        \draw [<-,>=latex,] (U) --  (X);
        \begin{scope}[transform canvas={xshift=-.15em}]
          \draw [->,>=latex,] (Y) --  (W);
        \end{scope}
         \begin{scope}[transform canvas={xshift=.15em}]
          \draw [<-,>=latex,] (Y) --  (W);
        \end{scope}	
         \begin{scope}[transform canvas={xshift=-.1em, yshift=.1em}]
          \draw [->,>=latex,] (Z) --  (W);
        \end{scope}
         \begin{scope}[transform canvas={xshift=.1em, yshift=-.1em}]
          \draw [<-,>=latex,] (Z) --  (W);
        \end{scope}	
        \draw[->,>=latex] (U) -- (Z);

	\draw[->,>=latex] (Y) to [out=-15,in=30, looseness=2] (Y);
	\draw[->,>=latex] (Z) to [out=-30,in=15, looseness=2] (Z);
	\draw[->,>=latex] (U) to [out=180,in=135, looseness=2] (U);
	\draw[->,>=latex] (W) to [out=0,in=45, looseness=2] (W);
	
	\end{tikzpicture}
 	\caption{}
	\label{fig:identifiable_4}
  \end{subfigure}
    \hfill 
   \begin{subfigure}{\myscale\textwidth}
	\begin{tikzpicture}[{black, circle, draw, inner sep=0}]
	\tikzset{nodes={draw,rounded corners},minimum height=0.6cm,minimum width=0.6cm}	
	\tikzset{anomalous/.append style={fill=easyorange}}
	\tikzset{rc/.append style={fill=easyorange}}
	
	\node[fill=red!30] (X) at (0*\mysecondscale,0*\mysecondscale) {$X$} ;
	\node[fill=blue!30] (Y) at (2*\mysecondscale,0*\mysecondscale) {$Y$};
	\node (W) at (2*\mysecondscale,1.5*\mysecondscale) {$W$};
	\node (U) at (0*\mysecondscale,1.5*\mysecondscale) {$U$};
	\node (Z) at (1*\mysecondscale,0.75*\mysecondscale) {$Z$};

	\draw[->,>=latex, line width=2pt] (X) -- (Y);
        \draw [<-,>=latex,] (U) --  (X);
        \begin{scope}[transform canvas={xshift=-.15em}]
          \draw [->,>=latex,] (Y) --  (W);
        \end{scope}
         \begin{scope}[transform canvas={xshift=.15em}]
          \draw [<-,>=latex,] (Y) --  (W);
        \end{scope}	
         \begin{scope}[transform canvas={xshift=-.1em, yshift=.1em}]
          \draw [->,>=latex,] (Z) --  (W);
        \end{scope}
         \begin{scope}[transform canvas={xshift=.1em, yshift=-.1em}]
          \draw [<-,>=latex,] (Z) --  (W);
        \end{scope}	
        \draw[->,>=latex] (Z) -- (U);

        \draw[->,>=latex] (X) to [out=195,in=150, looseness=2] (X);
	\draw[->,>=latex] (Y) to [out=-15,in=30, looseness=2] (Y);
	\draw[->,>=latex] (Z) to [out=-30,in=15, looseness=2] (Z);
	\draw[->,>=latex] (U) to [out=180,in=135, looseness=2] (U);
	\draw[->,>=latex] (W) to [out=0,in=45, looseness=2] (W);
	
	\end{tikzpicture}
 	\caption{}
	\label{fig:identifiable_5}
  \end{subfigure}
    \hfill 
   \begin{subfigure}{\myscale\textwidth}
	\begin{tikzpicture}[{black, circle, draw, inner sep=0}]
	\tikzset{nodes={draw,rounded corners},minimum height=0.6cm,minimum width=0.6cm}	
	\tikzset{anomalous/.append style={fill=easyorange}}
	\tikzset{rc/.append style={fill=easyorange}}
	
	\node[fill=red!30] (X) at (0*\mysecondscale,0*\mysecondscale) {$X$} ;
	\node[fill=blue!30] (Y) at (2*\mysecondscale,0*\mysecondscale) {$Y$};
	\node (W) at (2*\mysecondscale,1.5*\mysecondscale) {$W$};
	\node (U) at (0*\mysecondscale,1.5*\mysecondscale) {$U$};
	\node (Z) at (1*\mysecondscale,0.75*\mysecondscale) {$Z$};

	\draw[->,>=latex, line width=2pt] (X) -- (Y);
        \draw [<-,>=latex,] (U) --  (X);
        \begin{scope}[transform canvas={xshift=-.15em}]
          \draw [->,>=latex,] (Y) --  (W);
        \end{scope}
         \begin{scope}[transform canvas={xshift=.15em}]
          \draw [<-,>=latex,] (Y) --  (W);
        \end{scope}	
         \begin{scope}[transform canvas={xshift=-.1em, yshift=.1em}]
          \draw [->,>=latex,] (Z) --  (W);
        \end{scope}
         \begin{scope}[transform canvas={xshift=.1em, yshift=-.1em}]
          \draw [<-,>=latex,] (Z) --  (W);
        \end{scope}	
        \draw[->,>=latex] (Z) -- (U);

	\draw[->,>=latex] (Y) to [out=-15,in=30, looseness=2] (Y);
	\draw[->,>=latex] (Z) to [out=-30,in=15, looseness=2] (Z);
	\draw[->,>=latex] (U) to [out=180,in=135, looseness=2] (U);
	\draw[->,>=latex] (W) to [out=0,in=45, looseness=2] (W);
	
	\end{tikzpicture}
 	\caption{}
	\label{fig:identifiable_6}
  \end{subfigure}

     \caption{Examples of SCGs with $5$ vertices where the direct effect $\alpha_{X_{t-\gamma_{xy}},Y_t}$ is identifiable for all $\gamma_{xy}$. Red and blue vertices respectively represent the cause and the effect we are interested in and the thick edge corresponds to the the edge between them. All SCGs share the same skeleton, the edges $X\rightarrow Y$, $Y\leftrightarrows W$, and $Z\leftrightarrows W$ and the cycles of size $2$ on $Y, W,Z$ and $U$.}
    \label{fig:identifiable}
\end{figure*}
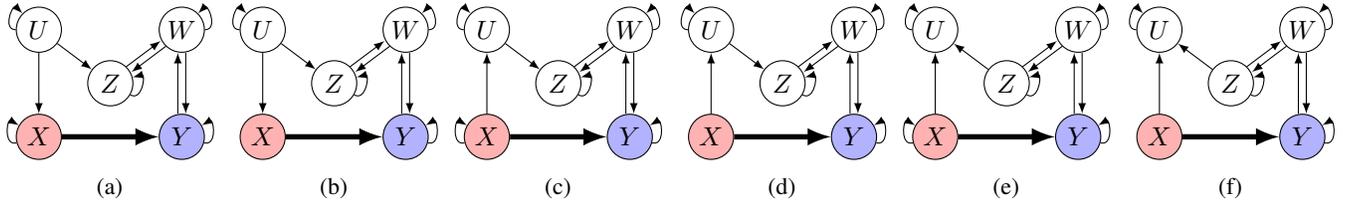%

To graphically illustrate the backward implication of Theorem~\ref{theorem:identifiability}, we consider three different examples which respectively correspond to Figures~\ref{fig:cond1},~\ref{fig:cond2a}, and \ref{fig:cond2b} and to Conditions~\ref{item:1},~\ref{item:2.1}, and \ref{item:2.2} of the theorem.

\begin{example}
\label{example:1}
Given the SCG in Figure~\ref{fig:cond1:SCG1}, $\alpha_{X_{t-1},Y_t}$ is non-identifiable since Condition~\ref{item:1} of  Theorem~\ref{theorem:identifiability} is satisfied. This can be illustrated by looking at the  two FTCGs in Figures~\ref{fig:cond1:FTCG1-1} and \ref{fig:cond1:FTCG1-2} which are compatible with the given SCG. In the first FTCG, it is obvious that it is important to adjust on $X_t$ in order to identify $\alpha_{X_{t-1},Y_t}$ (to block the path $\langle X_{t-1}, X_t, Y_t \rangle$). However, in the other FTCG, $X_t$ is a descendant of $Y_t$ therefore it is important not to adjust on it. Which means that without knowing which is the true FTCG and given only the SCG, $\alpha_{X_{t-1},Y_t}$ is non-identifiable. 
Notice that $\alpha_{X_{t-1},Y_t}$ remains non-identifiable if we remove the cycle of size $2$ on $X$ and we replace $X \leftarrow U$ by $X \leftrightarrows U$ since $X \leftrightarrows U$ induces a cycle of size $3$. This cycle means that there might exists an active path $\langle X_{t-1}\rightarrow U_t \rightarrow X_t \rightarrow Y_t\rangle$ in one FTCG and that $\{U_t, X_t\}$ can be descendants of $Y_t$ in another FTCG.

Similarly, given the SCG in Figure~\ref{fig:cond1:SCG2}, $\alpha_{X_{t},Y_t}$ is non-identifiable because in the first FTCG, at least one vertex in $\{U_t, Z_t, W_t\}$ should be in every valid adjustment set but in the second FTCG, these vertices are descendants of $Y_t$. 

\end{example}

\begin{example}
\label{example:2}
Given the SCG in Figure~\ref{fig:cond2a:SCG} (which is similar to the SCGs considered in Example~\ref{example:1} but where the cycle of size $2$ on $X$ is removed so Condition~\ref{item:1} of  Theorem~\ref{theorem:identifiability} is no longer satisfied), $\alpha_{X_{t},Y_t}$ is non-identifiable since Condition~\ref{item:2.1} of Theorem~\ref{theorem:identifiability} is  satisfied. This can be illustrated as in Example~\ref{example:1}.
Notice that in this case $\alpha_{X_{t-1},Y_t}$ becomes identifiable as the path $\langle X_{t-1}, X_t, Y_t \rangle$ can no longer exist.
\end{example}

\begin{example}
\label{example:3}
Given the SCG in Figure~\ref{fig:cond2b:SCG} (which is similar to the SCGs considered in Example~\ref{example:2} but where the orientation between $X$ and $U$ is reversed), $\alpha_{X_{t-1},Y_t}$ is non-identifiable since Condition~\ref{item:2.2} of Theorem~\ref{theorem:identifiability} is satisfied. 
This can be illustrated by looking at the  two FTCGs in Figures~\ref{fig:cond2b:FTCG1} and \ref{fig:cond2b:FTCG2}. In the first FTCG, at least one vertex in $\{U_t, Z_t, W_t\}$ should be in every valid adjustment set  but in the second FTCG, these vertices are
descendants of $Y_t$. 
Notice that we have the same result for $\alpha_{X_{t},Y_t}$.
\end{example}

We  give additional examples of SCGs in Figure~\ref{fig:identifiable} where the direct effect $\alpha_{X_{t-\gamma_{xy}},Y_t}$ is identifiable for all $\gamma_{xy}$.
Note that  $\alpha_{X_{t-\gamma_{xy}},Y_t}$ remains identifiable in Figures~\ref{fig:identifiable_1},\ref{fig:identifiable_2},\ref{fig:identifiable_3},\ref{fig:identifiable_4} even if we replace $X\rightarrow U$ or $X \leftarrow U$ by $X\leftrightarrows U$.

It is important to highlight that Theorem~\ref{theorem:identifiability} encompasses the identifiability result presented in \cite{Assaad_2023}. Specifically, when there are no cycles of size greater than $2$ in the SCG, conditions 1 and 2 of the theorem are not satisfied, indicating the identifiability of the direct effect.
Interestingly, Theorem~\ref{theorem:identifiability} also shows that under some conditions we can directly know that there exist at least one non-identifiable instantaneous direct effect. This is given by the following Corollary.

\begin{restatable}{corollary}{mycorollaryone}
\label{Cor:cycle_implies_non_identifiability}
    Let $\mathcal{G}_s = (\mathcal{V}_s, \mathcal{E}_s)$ be an SCG. If there exists a cycle in $\mathcal{G}_s$ of length strictly greater than $2$, then $\exists \alpha_{X_t,Y_t}$ which is non-identifiable.
\end{restatable}

\begin{proof}
    Let 
    $\langle V^1,\dots,V^n\rangle$ be a cycle in $\mathcal{G}_s$ with $n \geq 3$. $\alpha_{V^1_t,V^2_t}$ from $V^1_t$ to $V^2_t$ is non-identifiable because the path $\langle V^n,V^{n-1},\dots,V^2\rangle$ from $V^1=V^n$ to $V^2$ in the SCG, verifies Condition~\ref{item:2.1} of Theorem~\ref{theorem:identifiability}.
\end{proof}


\section{Two Sound Adjustment Sets}
\label{sec:two_adjustment_sets}

In this section, we provide two finite adjustment sets that can be used to estimate the direct effect whenever it is identifiable. In an effort to be succinct we only give a sketch for the proof of the soundness of the second adjustment. For the full proof, please refer to the supplementary material.


For a given FTCG $\mathcal{G}_f$, to estimate $\alpha_{X_{t-\gamma_{xy}}, Y_t}$ from data, it is necessary and sufficient to adjust on a finite set $\mathcal{Z}_{f}$ such that $\mathcal{Z}_{f} \cap (Desc(Y_t,\mathcal{G}_f) \cup \{X_{t-\gamma_{xy}}, Y_t\}) = \emptyset$ which blocks every non-direct path from $X_{t-\gamma_{xy}}$ to $Y_{t}$.
Thus, given an SCG $\mathcal{G}_s$,  one needs to find a set $\mathcal{Z}_{f}$ such that $\mathcal{Z}_{f} \cap (Desc(Y_t,\mathcal{G}_f) \cup \{X_{t-\gamma_{xy}}, Y_t\}) = \emptyset$ which blocks every non-direct path from $X_{t-\gamma_{xy}}$ to $Y_{t}$ in every FTCG $\mathcal{G}_f$ compatible with $\mathcal{G}_s$ of maximal lag at most $\gamma_{max}$.


The following corollary formally indicates the soundness of the finite adjustment set defined in Definition~\ref{def:huge-single-door_set}.


\begin{restatable}{corollary}{mycorollarytwo}
\label{Cor:soundness_of_biggest_single-door_set}
    Let $\mathcal{G}_s=(\mathcal{V}_s, \mathcal{E}_s)$ be an SCG, $\gamma_{max} \geq 0$ a maximum lag. Consider two vertices $X$ and $Y$ such that $X\in Par(Y,\mathcal{G}_s)$ and $\alpha_{X_{t-\gamma_{xy}},Y_t}$ is identifiable following Theorem~\ref{theorem:identifiability}.
    Then $\mathcal{Z}_{f}^{\text{Def~\ref{def:huge-single-door_set}}}=\mathcal{A}_{\leq t} \cup \mathcal{D}_{<t}$ as defined in Definition~\ref{def:huge-single-door_set}  is a valid adjustment set for $\alpha_{X_{t-\gamma_{xy}}, Y_t}$.
\end{restatable}

\begin{proof}
    The proof of the forward implication of Theorem~\ref{theorem:identifiability} proves this corollary.
\end{proof}

Note that many valid adjustment sets may exist. An estimator of direct effects based on any of these sets is unbiased, but the estimation variance may vary for different sets.
Thus, it is interesting to search for many, and ideally all, such sets in order to optimize the estimation of direct effects. In the following, we give a smaller valid adjustment set. 

\begin{definition}[A Second Finite Adjustment Set]
    \label{def:single-door_set}
    Consider an SCG $\mathcal{G}_s=(\mathcal{V}_s,\mathcal{E}_s)$, a maximal lag $\gamma_{max}$, two vertices $X$ and $Y$ with $X\in Par(Y,\mathcal{G}_s)$ and a lag $\gamma_{xy}$.
    $\mathcal{Z}_{f}^{\text{Def~\ref{def:single-door_set}}} = \mathcal{D}^{Anc(Y)}_{<t} \cup \mathcal{A}^{Anc(Y)}_{\leq t}$ is an adjustment set relative to $(X_{t-\gamma_{xy}}, Y_t)$ such that:
    \begin{align*}
        \mathcal{D}^{Anc(Y)}_{<t} &= \{V_{t'} &|& V \in Anc(Y,\mathcal{G}_s) \cap Desc(Y,\mathcal{G}_s),& 
        \\ & & &~t-\gamma_{max} \leq t' < t\} \backslash \{X_{t-\gamma_{xy}}\}& \text{ and}\\
        \mathcal{A}^{Anc(Y)}_{\leq t} &= \{V_{t'} &|& V \in Anc(Y,\mathcal{G}_s) \backslash Desc(Y,\mathcal{G}_s),&
        \\ & & & ~t-\gamma_{max} \leq t' \leq t\} \backslash \{X_{t-\gamma_{xy}}\}.&
    \end{align*}
\end{definition}

The following proposition formally indicates the soundness of the finite adjustment set defined in Definition~\ref{def:single-door_set}.

\begin{restatable}{proposition}{mypropositionone}
    \label{prop:soundness_singledoor}
    Let $\mathcal{G}_s=(\mathcal{V}_s, \mathcal{E}_s)$ be an SCG, $\gamma_{max} \geq 0$ a maximum lag. Consider two vertices $X$ and $Y$ such that $X\in Par(Y,\mathcal{G}_s)$ and $\alpha_{X_{t-\gamma_{xy}},Y_t}$ is identifiable following Theorem~\ref{theorem:identifiability}.
    Then $\mathcal{Z}_{f}^{\text{Def~\ref{def:single-door_set}}}$ as defined in Definition~\ref{def:single-door_set}  is a valid adjustment set for $\alpha_{X_{t-\gamma_{xy}}, Y_t}$.
\end{restatable}
\begin{sproof}
Corollary~\ref{Cor:soundness_of_biggest_single-door_set} shows that the adjustment set $\mathcal{Z}^{\text{Def~\ref{def:huge-single-door_set}}}_f$ allows to estimate the direct effect $\alpha_{X_{t-\gamma_{xy}}, Y_t}$.
Firstly, since $\mathcal{Z}^{\text{Def~\ref{def:single-door_set}}}_f \subseteq \mathcal{Z}^{\text{Def~\ref{def:huge-single-door_set}}}_f$ and $\mathcal{Z}^{\text{Def~\ref{def:huge-single-door_set}}}_f \cap (Desc(Y_t,\mathcal{G}_f) \cup \{X_{t-\gamma_{xy}}\}) = \emptyset$, it is immediate that $\mathcal{Z}^{\text{Def~\ref{def:single-door_set}}}_f \cap (Desc(Y_t,\mathcal{G}_f) \cup \{X_{t-\gamma_{xy}}\}) = \emptyset$.
Moreover, it is known that adjusting on ancestors of the cause $Y_t$ is sufficient to estimate the direct effect in directed acyclic graphs.
Therefore, it is intuitive that the restriction of $\mathcal{Z}^{\text{Def~\ref{def:huge-single-door_set}}}_f$ to ancestors of $Y_t$, $\mathcal{Z}^{\text{Def~\ref{def:huge-single-door_set}}}_f \cap \{V_{t'}| V \in Anc(Y,\mathcal{G}_s),~t'\leq t\} = \mathcal{Z}^{\text{Def~\ref{def:single-door_set}}}_f$, allows to estimate the direct effect $\alpha_{X_{t-\gamma_{xy}}, Y_t}$.
The full proof is given in the supplementary materials.
\end{sproof}

\section{Conclusion}
\label{sec:conc}
In this paper, we developed a new graphical criteria for the identifiability of direct effects in linear dynamic structural causal models from summary causal graphs.
Theorem~\ref{theorem:identifiability} has important ramifications to the theory and practice of observational studies in dynamic systems. It implies that the key to graphical identifiability of the direct effect of $X_{t-\gamma_{xy}}$ on $Y_t$ from summary causal graphs lies not only in finding a set of non-descendants of $Y$ in the summary causal graph that are able of blocking paths between $X$ and $Y$ but also in some descendants of $Y$ in the case $\gamma_{xy}>0$. 
Furthermore, in case of identifiability, we presented two adjustments sets that can be used to estimate the direct effects from data.

The finding of this paper should be useful for many applications such as root cause identification in dynamic systems and it should open new research questions.
Namely, for future works, it would be interesting to have a criterion along with a completeness result describing every possible adjustment set.
In addition, 
since in many real world applications causal relations can be nonlinear, it would be interesting to extend this work to nonlinear SCMs and consider non-parametric direct effects \citep{Robins_1992,Pearl_2001}.
Finally, as many other works, we assumed that the FTCG is acyclic but we think that this assumption can be relaxed, so it would be interesting to formally check the validity of our results for cyclic FTCGs.


\appendix
\section{Technical Appendix}

In Section~"\textit{Further Necessary Definitions}" 
we give several definitions that are needed for a more detailed proof of Theorem~\ref{theorem:identifiability}. In Section~"\textit{Proofs}" 
we start by given several lemmas and properties along with their proofs and then we use them to prove Theorem~\ref{theorem:identifiability}.
Finally, at the end of the Section, 
we give the proof 
of Proposition~\ref{prop:soundness_singledoor}.

\subsection{Further Necessary Definitions}
\label{sec:appendix:extra_def}


In order to prove Theorem~\ref{theorem:identifiability}, one needs to further define walks in SCGs, adapt the concept of blocked paths for walks and introduce some new notions related to paths and walks.

\begin{definition}[Walks and Paths in SCGs, adaptation of Definition~\ref{def:Paths_SCG}]
    \label{def:Walks_Paths_SCG}
    A \emph{walk} between two vertices $X$ to $Y$ is an ordered sequence of vertices denoted as $\pi_s=\langle V^1,\dots,V^n \rangle$ such that $V^1=X$, $V^n=Y$ and $\forall 1\leq i < n$, $V^i$ and $V^{i+1}$ are adjacent (\ie, $V^i\rightarrow V^{i+1}\text{ or }V^i\leftarrow V^{i+1}\text{ or }V^i\rightleftarrows V^{i+1}$).
    In this paper, a walk $\pi_s$ between $X$ to $Y$ is said to be \emph{non-direct} if $\pi_s \neq \langle X \rightarrow Y \rangle$.
    A path is a walk with no two identical vertices.
\end{definition}

\begin{definition}[Blocked Walk in SCGs, adaptation of Definition~\ref{def:BlockedPathInSCG}]
    \label{def:BlockedWalkInSCG}
    In a SCG $\mathcal{G}_s=(\mathcal{V}_s,\mathcal{E}_s)$, a walk $\pi_s=\langle V^1,\dots,V^n \rangle$ is said to be \emph{blocked} by a set of vertices $\mathcal{Z}_s\subseteq\mathcal{V}_s$ if:
    \begin{enumerate}
        \item \label{item:ActivelyBlockedInSCG} $\exists 1 < i < n \st V^{i-1}\leftarrow V^i$ or $V^i\rightarrow V^{i+1}$ and $V^i\in\mathcal{Z}_s$, or
        \item \label{item:PassivelyBlockedInSCG} $\exists 1 < i \leq j < n \st V^{i-1}\rightarrow V^i\rightleftarrows\dots\rightleftarrows V^{j} \leftarrow V^{j+1}\text{ and }Desc(V^i, \mathcal{G}_s)\cap\mathcal{Z}_s=\emptyset$.
    \end{enumerate}
    A walk which is not blocked is said to be \emph{active}.
    When the set $\mathcal{Z}_s$ is not specified, it is implicit that we consider $\mathcal{Z}_s=\emptyset$.
    In the case of condition~\ref{item:ActivelyBlockedInSCG}, we say that $\pi_s$ is manually $\mathcal{Z}_s$-blocked by $V^i$ and in the case of condition~\ref{item:PassivelyBlockedInSCG} we say that $\pi_s$ is passively $\mathcal{Z}_s$-blocked by $\{V^k|i\leq k \leq j\}$.
\end{definition}

As SCGs represent FTCGs, the walks in a SCG can represent the paths of compatible FTCGs.
This is gives rise to the notion of compatible walk.

\begin{definition}[Compatible Walk]
    \label{def:CompatibleWalk}
     Let $\mathcal{G}_{f}=(\mathcal{V}_{f},\mathcal{E}_{f})$ be a FTCG and $\mathcal{G}_s=(\mathcal{V}_{s},\mathcal{E}_{s})$ the compatible SCG.
     A path $\pi_f=\langle V^1_{t_1},\dots,V^n_{t_n} \rangle$ in $\mathcal{G}_{f}$ can be uniquely represented as a walk $\pi_s=\langle V^1,\dots,V^n \rangle$ in $\mathcal{G}_s$ in which the temporal information has been removed.
     We refer to $\pi_s$ as $\pi_f$'s \emph{compatible walk} and we write $\pi_s=\phi(\pi_f)$.
     \eg, $\phi(\langle X_{t-1},X_{t},Y_t,Z_t,Z_{t+1},X_{t+1} \rangle)=\langle X,X,Y,Z,Z,X \rangle$.
\end{definition}

Furthermore, since there exists an infinite number of walks in a SCG, it is hard in practice to say verify anything about walks. Therefore, it is necessary to have a notion which creates a link between walks and paths of a SCG. This is the purpose of the following notion of primary path.

\begin{definition}[Primary path]
    Let $\mathcal{G}_s=(\mathcal{V}_{s},\mathcal{E}_{s})$ be a SCG and $\pi_s=\langle V^1,\dots,V^n \rangle$ a walk  from $X$ to $Y$.
    $\pi'_s = \langle U^1,\dots,U^m \rangle$ such that $U^1 = V^1$ and $U^{k+1} = V^{max\{i|V^i = U^k\}+1}$ is called the primary path of $\pi_s$.
\end{definition}

Lastly, in the following, we write $$t_{min}(\pi_f=\langle V^1_{t^1},\dots,V^n_{t^n} \rangle)=min\{t^i|1\leq i \leq n\}$$ and $$t_{max}(\pi_f=\langle V^1_{t^1},\dots,V^n_{t^n} \rangle)=max\{t^i|1\leq i \leq n\}.$$

\subsection{Proofs}
\label{sec:appendix:proofs}

In this section, we prove the identifiability result. 
We start by stating a trivial lemma that will be needed for the general identifiability result.

\begin{restatable}{lemma}{mylemmaone}
    \label{lem:identifiability:XnotinParentsY}
    Let $\mathcal{G}_s=(\mathcal{V}_s, \mathcal{E}_s)$ be a SCG, $\gamma_{max} \geq 0$ a maximum lag and $\alpha_{Y_t,X_{t-\gamma_{xy}}}$ the direct effect of $X_{t-\gamma_{xy}}$ on $Y_t$ such that $X, Y \in \mathcal{V}_s$, $X \neq Y$ and $0 \leq \gamma_{xy} \leq \gamma_{max}$.
    If $X \notin Par(Y,\mathcal{G}_s)$ then $\alpha_{Y_t, X_{t-\gamma_{xy}}}$ is identifiable.
\end{restatable}

\begin{proof}
    Suppose $ X \notin Par(Y,\mathcal{G}_s)$.
    Then $X_{t-\gamma_{xy}} \notin Par(Y_t,\mathcal{G}_f)$ and the direct effect is equal to zero (\ie, $\alpha_{Y_t,X_{t-\gamma_{xy}}}=0$).
\end{proof}

The aim of this work is to identify direct effects that can be estimated from data. This implies that we are interested in figuring out if for a given $X_{t-\gamma_{xy}}$ and a given $Y_t$, it is possible to find at least some \emph{finite} adjustment set that removes all confounding bias and non-direct effects and that does not create any selection bias between $X_{t-\gamma_{xy}}$ and $Y_t$. Therefore, in the following lemma, we show that what we are trying to achieve is possible by pointing out that infinite sets are not necessary to block paths between $X_{t-\gamma_{xy}}$ and $Y_t$ in a given FTCG.

\begin{restatable}{lemma}{mylemmatwo}
    \label{lem:blocked_when_tmax>t_or_tmin<t-gammaxy-gammamax}
    Let $\mathcal{G}_f=(\mathcal{V}_f, \mathcal{E}_f)$ be a FTCG of maximal lag at most $\gamma_{max} \geq 0$ and $\alpha_{Y_t,X_{t-\gamma_{xy}}}$ the direct effect of $X_{t-\gamma_{xy}}$ on $Y_t$ such that $X_{t-\gamma_{xy}}, Y_t \in \mathcal{V}_f$, $X \neq Y$ and $0 \leq \gamma_{xy} \leq \gamma_{max}$.
    Let $\pi_f=\langle V^1_{t^1},\dots,V^n_{t^n} \rangle$ a path from $X_{t-\gamma_{xy}}$ to $Y_t$ in $\mathcal{G}_f$.
    If $t_{max}(\pi_f)>t$ then $\pi_f$ is passively blocked by any $\mathcal{Z}_{f}\subseteq \mathcal{V}_{f}$ such that $\mathcal{Z}_{f} \cap \{V_{t'} \in \mathcal{V}_{f}| t' > t\} = \emptyset$.
    If $t_{min}(\pi_f)<t-\gamma_{xy}$ then $\pi_f$ is manually blocked by any $\mathcal{Z}_{f}\subseteq \mathcal{V}_{f}$ such that $\{V_{t'} \in \mathcal{V}_{f}| t-\gamma_{max} \leq t' < t\} \backslash\{X_{t-\gamma_{xy}}\} \subseteq \mathcal{Z}_{f}$.
\end{restatable}

\begin{proof}
     Let $\mathcal{G}_f=(\mathcal{V}_{f},\mathcal{E}_{f})$ be a FTCG, $X_{t-\gamma_{xy}} \neq Y_t \in \mathcal{V}_{f}$ and $\pi_f=\langle V^1_{t^1},\dots,V^n_{t^n} \rangle$ a path from $X_{t-\gamma_{xy}}$ to $Y_t$ in $\mathcal{G}_f$.
    \begin{itemize}
        \item Suppose $t_{max}(\pi_f)>t$.
        Since $t-\gamma_{xy} \leq t < t_{max}(\pi_f)$ there exists $1<i\leq j<n$ such that $t^{i-1} < t^{i}$, $t^{j} > t^{j+1}$ and $\forall i\leq k \leq j, t^{k} = t_{max}(\pi_f)$.
        Therefore, $V^{i-1}_{t^{i-1}} \rightarrow V^i_{t^{i}}$ and $V^j_{t^{j}} \leftarrow V^{j+1}_{t^{j+1}}$ in $\pi_f$.
        Thus, $\exists i \leq k \leq j \st V^{k-1}_{t^{k-1}} \rightarrow V^k_{t^{k}} \leftarrow V^{k+1}_{t^{k+1}}$ in $\pi_f$  and $t^k = t_{max}(\pi_f) > t$.
        In conclusion, $\pi_f$ is passively blocked by any $\mathcal{Z}_{f}\subseteq \mathcal{V}_{f} \st Desc(V^k_{t^{k}},\mathcal{G}_f) \cap \mathcal{Z}_{f} = \emptyset$ so by any $\mathcal{Z}_{f} \st \mathcal{Z}_{f} \cap \{V_{t'} \in \mathcal{V}_{f}| t' > t\} = \emptyset$
        \item Suppose $t_{min}(\pi_f)<t-\gamma_{xy}$.
        Since $t_{min}(\pi_f)<t-\gamma_{xy}\leq t$ and $t^n = t$ there exists $1<i<n$ such that $t^{i} < t \leq t^{i+1}$.
        Therefore, $V^{i}_{t^{i}} \rightarrow V^{i+1}_{t^{i+1}}$ in $\pi_f$ and $t-\gamma_{max} \leq t^i < t \leq t^{i+1}$ so $V^i_{t^i} \in \{V_{t'} \in \mathcal{V}_{f}| t-\gamma_{max} \leq t' < t\} \backslash\{X_{t-\gamma_{xy}}\}$.
        In conclusion, $\pi_f$ is manually blocked by any $\mathcal{Z}_{f}\subseteq \mathcal{V}_{f} \st V^i_{t^{i}} \in \mathcal{Z}_{f}$ so by any $\mathcal{Z}_{f} \st \{V_{t'} \in \mathcal{V}_{f}| t-\gamma_{max} \leq t' < t\} \backslash\{X_{t-\gamma_{xy}}\} \subseteq \mathcal{Z}_{f}$.
     \end{itemize}
\end{proof}

Lemma~\ref{lem:identifiability:XnotinParentsY} and Lemma~\ref{lem:blocked_when_tmax>t_or_tmin<t-gammaxy-gammamax} respectively show that the case where $X \notin Par(Y,\mathcal{G}_s)$ is trivially identifiable and  that any path $\pi_f$ where $t_{min}(\pi_f)<t-\gamma_{xy}$ and $t_{max}(\pi_f)>t$ are easily blocked by $\mathcal{Z}_f$ as defined in Definition~\ref{def:huge-single-door_set}. Thus we will consider $X \in Par(Y,\mathcal{G}_s)$ and $t-\gamma_{xy} \leq t_{min}(\pi_f) \leq t_{max}(\pi_f)\leq t$ in the following Lemmas.

In these cases, one might think that to block all active non-direct paths between $X_{t-\gamma_{xy}}$ and $Y_t$, it is simply sufficient to adjust on all vertices in the FTCG which do not temporally succeed the effect $Y_t$ and have compatible vertices on an active path between $X$ and $Y$ in the SCGs. 
However, in general this is not and this is given by  the following lemma.

\begin{restatable}{lemma}{mylemmathree}
    \label{lem:identifiability:BlocksWalksNotinDescendantsY}
    Let $\mathcal{G}_s=(\mathcal{V}_s, \mathcal{E}_s)$ be a SCG, $\gamma_{max} \geq 0$ a maximum lag and $\alpha_{Y_t,X_{t-\gamma_{xy}}}$ the direct effect of $X_{t-\gamma_{xy}}$ on $Y_t$ such that $X, Y \in \mathcal{V}_s$, $X \neq Y$ and $0 \leq \gamma_{xy} \leq \gamma_{max}$.
    For every non-direct walk $\pi_s=\langle V^1,\dots,V^n \rangle$ between $X$ and $Y$ such that $\langle V^2,\dots,V^{n-1} \rangle \not\subseteq Desc(Y, \mathcal{G}_s)$\footnote{For simplification, we sometimes abuse the notation of walks. Here $\langle V^2,\dots,V^{n-1} \rangle = \{V^i|2\leq i \leq n-1\}$ can be empty if $n\leq 2$.}, every compatible path $\pi_f\in \phi^{-1}(\pi_s)$ from $X_{t-\gamma_{xy}}$ to $Y_t$ can be blocked by the adjustment set $\mathcal{Z}_{f} = \mathcal{A}_{\leq t} \cup \mathcal{D}_{<t}$ defined in Definition~\ref{def:huge-single-door_set}.
\end{restatable}

\begin{proof}
    Suppose there exists a non-direct walk $\pi_s=\langle V^1,\dots,V^n \rangle$ between $X$ and $Y$ such that $\langle V^2,\dots,V^{n-1} \rangle \not\subseteq Desc(Y, \mathcal{G}_s)$.
    Then, take $\pi_f=\langle V^1_{t^1},\dots,V^n_{t^n} \rangle$ from $X_{t-\gamma_{xy}}$ to $Y_t$ compatible with $\pi_s$ (\ie, $\pi_f\in\phi^{-1}(\pi_s)$).
    Take $j=max\{1< i < n | V^i \notin Desc(Y, \mathcal{G}_s) \}$.
    Notice that $V^j \notin Desc(Y, \mathcal{G}_s)$ and $V^{j+1} \in Desc(Y, \mathcal{G}_s)$ so $V^j_{t^j} \rightarrow V^{j+1}_{t^{j+1}} \in \pi_f$.
    Therefore, since $t-\gamma_{xy} \leq t_{min}(\pi_f) \leq t_{max}(\pi_f)\leq t$ by Lemma~\ref{lem:blocked_when_tmax>t_or_tmin<t-gammaxy-gammamax}, $t-\gamma_{xy} \leq t^j \leq t$ and thus $\pi_f$ is manually blocked by $V^j_{t^j} \in \mathcal{A}_{\leq t} \subseteq \mathcal{Z}_{f}$.
\end{proof}

One important factor that we did not take into account in these first lemmas is the value of $\gamma_{xy}$. Indeed, when $\gamma_{xy}>0$ it is safe to say that the problem should become easier as in this case we know that the parents of $X_{t-\gamma_{xy}}$ cannot be descendants of $Y_t$. 
Therefore, distinguishing the case where $\gamma_{xy}=0$ and the case where $\gamma_{xy}=0$ is important to reach a general identifiability result.

\begin{restatable}{lemma}{mylemmafour}
    \label{lem:identifiability:BlocksWhenGamma>0}
    Let $\mathcal{G}_s=(\mathcal{V}_s, \mathcal{E}_s)$ be a SCG, $\gamma_{max} \geq 0$ a maximum lag and $\alpha_{Y_t,X_{t-\gamma_{xy}}}$ the direct effect of $X_{t-\gamma_{xy}}$ on $Y_t$ such that $X, Y \in \mathcal{V}_s$, $X \neq Y$ and $0 < \gamma_{xy} \leq \gamma_{max}$.
    For every non-direct walk $\pi_s=\langle V^1,\dots,V^n \rangle$ from $X$ to $Y$ such that $\exists 1 \leq i < n,~V^i \leftarrow V^{i+1}$ (\ie, not $V^i \rightarrow V^{i+1}$ and not $V^i \rightleftarrows V^{i+1}$), every compatible path $\pi_f\in \phi^{-1}(\pi_s)$ from $X_{t-\gamma_{xy}}$ to $Y_t$ can be blocked by $\mathcal{Z}_{f} = \mathcal{A}_{\leq t} \cup \mathcal{D}_{<t}$ because $\mathcal{Z}_{f} \cap \mathcal{D}_{\geq t} = \emptyset$ where:
    \begin{itemize}
        \item $\mathcal{A}_{\leq t},~\mathcal{D}_{<t}$ are defined in Definition~\ref{def:huge-single-door_set} and
        \item $\mathcal{D}_{\geq t}$ is the set of instants of descendants of $Y$ greater or equal to $t$, \ie, $\mathcal{D}_{\geq t} = \{V_{t'} | V \in Desc(Y, \mathcal{G}_s),~t' \geq t\}$.
    \end{itemize}
    Note that $\pi_s=\langle V^1,\dots,V^n \rangle$ from $X$ to $Y$ is non-direct, $X\in Par(Y,\mathcal{G}_s)$ by Lemma~\ref{lem:identifiability:XnotinParentsY} and  $\exists 1 \leq i < n,~V^i \leftarrow V^{i+1}$ implies that $n\geq3$.
\end{restatable}

\begin{proof}
    Let $\gamma_{xy}>0$ and $\pi_s=\langle V^1,\dots,V^n \rangle$ be a non-direct walk between $X$ and $Y$ such that $\exists 1 \leq i < n,~V^i \leftarrow V^{i+1}$.
    Then, take a path $\pi_f=\langle V^1_{t^1},\dots,V^n_{t^n} \rangle$ from $X_{t-\gamma_{xy}}$ to $Y_t$ compatible with $\pi_s$ (\ie, $\pi_f\in\phi^{-1}(\pi_s)$) such that $t_{min}(\pi_f) \geq t - \gamma_{xy}$ and $t_{max}(\pi_f)\leq t$ by Lemma~\ref{lem:blocked_when_tmax>t_or_tmin<t-gammaxy-gammamax} and take $1\leq i < n \st V^i_{t^i} \leftarrow V^{i+1}_{t^{i+1}}$.
    If $\langle V^2,\dots,V^n \rangle\not\subseteq Desc(Y,\mathcal{G}_s)$ then Lemma~\ref{lem:identifiability:BlocksWalksNotinDescendantsY} states that $\pi_f$ is $\mathcal{Z}_f$-blocked.
    Therefore, on can suppose $\langle V^2,\dots,V^n \rangle\subseteq Desc(Y,\mathcal{G}_s)$ and in particular $V^{i+1}\in Desc(Y,\mathcal{G}_s)$.
    If $t^{i+1} < t$ then $V^{i+1}_{t^{i+1}} \in \mathcal{D}_{< t} \subseteq \mathcal{Z}_f$ and $\pi_f$ is manually blocked by $V^{i+1}_{t^{i+1}} \in \mathcal{Z}_f$.
    Else,  $V^{i+1}_{t^{i+1}} \in \mathcal{D}_{\geq t}$. Since $\gamma_{xy} >0$, one can take $j=max\{1 < j \leq i | V^{j-1}_{t^{j-1}} \rightarrow V^j_{t^j}\}$.
    Notice that $V^j_{t^j}$ is a collider and that $V^j_{t^j} \in Desc(V^i_{t^i},\mathcal{G}_f)$ so $V^j_{t^j} \in \mathcal{D}_{\geq t}$ and $Desc(V^j_{t^j},\mathcal{G}_s) \subseteq \mathcal{D}_{\geq t}$.
    Therefore, $\pi_f$ is passively blocked by $V^{j}_{t^{j}} \in \mathcal{Z}_f$.
\end{proof}

Now we give the complementary of Lemma~\ref{lem:identifiability:BlocksWhenGamma>0} in the case where $X\rightleftarrows Y$.

\begin{restatable}{lemma}{mylemmafive}
    \label{lem:identifiability:BlocksWhenGamma>0Andn=2}
    Let $\mathcal{G}_s=(\mathcal{V}_s, \mathcal{E}_s)$ be a SCG, $\gamma_{max} \geq 0$ a maximum lag and $\alpha_{Y_t,X_{t-\gamma_{xy}}}$ the direct effect of $X_{t-\gamma_{xy}}$ on $Y_t$ such that $X, Y \in \mathcal{V}_s$, $X \neq Y$ and $0 < \gamma_{xy} \leq \gamma_{max}$.
    For every walk $\pi_s=\langle V^1,\dots,V^n \rangle$ from $X$ to $Y$ where $\exists 1 < i < n,~V^i = Y$, every compatible path $\pi_f\in \phi^{-1}(\pi_s)$ from $X_{t-\gamma_{xy}}$ to $Y_t$ can be blocked by $\mathcal{Z}_{f} = \mathcal{A}_{\leq t} \cup \mathcal{D}_{<t}$ (notice $\mathcal{Z}_{f} \cap \mathcal{D}_{\geq t} = \emptyset$) where $\mathcal{A}_{\leq t},~\mathcal{D}_{<t}$ are defined in Definition~\ref{def:huge-single-door_set} and $\mathcal{D}_{\geq t}$ is defined in Lemma~\ref{lem:identifiability:BlocksWhenGamma>0}.
\end{restatable}

\begin{proof}
    Let $\pi_s=\langle V^1,\dots,V^n \rangle$ be a walk between $X$ and $Y$ as described.
    Then, take $\pi_f=\langle V^1_{t^1},\dots,V^n_{t^n} \rangle$ from $X_{t-\gamma_{xy}}$ to $Y_t$ compatible with $\pi_s$ (\ie, $\pi_f\in\phi^{-1}(\pi_s)$) and $1<i<n \st V^{i} = Y$.
    Since $\pi_f$ is a path, $t^{i} \neq t$ so by Lemma~\ref{lem:blocked_when_tmax>t_or_tmin<t-gammaxy-gammamax}, $t-\gamma_{xy} \leq t^{i} < t$.
    If $\leftarrow V^{i}_{t^{i}}$ or $V^{i}_{t^{i}} \rightarrow$ then $\pi_f$ is manually blocked by $V^{i}_{t^{i}} \in \mathcal{D}_{<t} \subseteq \mathcal{Z}_f$.
    If $\rightarrow V^{i}_{t^{i}} \leftarrow V^{i+1}_{t^{i+1}}$ then $t^{i+1}\leq t^i < t$ and $\pi_f$ is a path so $V^{i+1}_{t^{i+1}} \neq X_{t-\gamma_{xy}}$.
    Therefore, $\pi_f$ is manually blocked by $V^{i+1}_{t^{i+1}} \in \mathcal{Z}_f$.
\end{proof}

Since Lemmas~\ref{lem:identifiability:BlocksWalksNotinDescendantsY},\ref{lem:identifiability:BlocksWhenGamma>0} and \ref{lem:identifiability:BlocksWhenGamma>0Andn=2} consider walks in SCGs but it is computationally easier to consider paths in SCGs, in the following we provide a list of properties to reconcile these two notions.

Let $\pi_s$ be a walk and $\pi'_s$ its primary path. They verify the following properties.

\begin{restatable}{property}{mypropertyone}
        \label{prop:primpath1} If $\pi'_s$ is passively blocked by $U^i$ then $\pi_s$ is passively blocked by at least a descendant of $U^i$.
\end{restatable}

\begin{proof}
    Suppose $\pi'_s$ is passively blocked by $U^i$ then there exists $i_1 \leq i \leq i_2$ such that $\pi'_s = \langle\dots \rightarrow U^{i_1} \rightleftarrows \dots \rightleftarrows U^i \rightleftarrows \dots \rightleftarrows U^{i_2} \leftarrow \dots \rangle$.
    Thus there exists $j_1 \leq j \leq j_2$ such that $V^{j_1}=U^{i_1}$, $V^{j} = U^{i}$, $V^{j_2}=U^{i_2}$ and $\pi_s = \langle \dots \rightarrow V^{j_1} \dots V^{j_2} \leftarrow \dots \rangle$.
    Let $k_1 = max\{j_1\leq k \leq j | V^{k-1}\rightarrow V^{k}\}$ and $k_2 = min\{j\leq k \leq j_2 | V^{k}\leftarrow V^{k+1}\}$.
    Notice that $V^{k_1},V^{k_2} \in Desc(V^j,\mathcal{G}_s)$ and since $V^{j} = U^{i}$, this corresponds to $V^{k_1},V^{k_2} \in Desc(U^i,\mathcal{G}_s)$.
    Moreover, at least one of $V^{k_1}$ and $V^{k_2}$ is a collider in $\pi_s$.
    Indeed, if all edges between $V^{k_1}$ and $V^{k_2}$ are $\rightleftarrows$ then both $V^{k_1}$ and $V^{k_2}$ are colliders.
    If $\exists k_1\leq k < k_2$ such that $V^k\rightarrow V^{k+1}$ then by definition of $k_1$, one can deduct that $k\geq j$ and $V^{k_2}$ is a collider.
    Similarly, if $\exists k_1< k \leq k_2$ such that $V^{k-1} \leftarrow V^{k}$ then by definition of $k_2$, one can deduct that $k\leq j$ and $V^{k_1}$ is a collider.
    Thus, $\pi_s$ is passively blocked by at least a descendant of $U^i$.
\end{proof}

\begin{restatable}{property}{mypropertytwo}
        \label{prop:primpath2} If $m=2$ then either $n=2$ or $\exists 1 < i < n \st V^i=X$ or $V^i = Y$.
\end{restatable}

\begin{proof}
    If $m=2$ then $V^{max\{i|V^i=V^1\}+1} = V^n$.
    Firstly, if $max\{i|V^i=V^1\}+1 = n$ then either $n=2$ or $i=max\{i|V^i=V^1\}$ verifies $1<i<n$ and $V^{i} = V^1 = X$.
    Secondly, if $max\{i|V^i=V^1\}+1 < n$ then $i=max\{i|V^i=V^1\}+1$ verifies $1<i<n$ and $V^{i} = V^n = Y$.
\end{proof}

\begin{restatable}{property}{mypropertythree}
        \label{prop:primpath3} If $\langle U^2,\dots,U^{m-1} \rangle \not\subseteq Desc(Y, \mathcal{G}_s)$ then $\langle V^2,\dots,V^{n-1} \rangle \not\subseteq Desc(Y, \mathcal{G}_s)$.
\end{restatable}

\begin{proof}
    Since $\{U^1,\dots,U^m\} \subseteq \{V^1,\dots,V^n\}$, $\langle U^2,\dots,U^{m-1} \rangle \not\subseteq Desc(Y, \mathcal{G}_s) \implies \langle V^2,\dots,V^{n-1} \rangle \not\subseteq Desc(Y, \mathcal{G}_s)$.
\end{proof}

\begin{restatable}{property}{mypropertyfour}
        \label{prop:primpath4} If $\exists 1 \leq i < m,~U^i \leftarrow U^{i+1}$ then $\exists 1 \leq i < n,~V^i \leftarrow V^{i+1}$.
\end{restatable}

\begin{proof}
    $\forall 1 \leq i < m,~\exists 1\leq j < n$ such that $V^j = U^i$ and $V^{j+1} = U^{i+1}$, therefore $\exists 1 \leq i < m,~U^i \leftarrow U^{i+1} \implies \exists 1 \leq j < n,~V^j \leftarrow V^{j+1}$.
\end{proof}

Lemma~\ref{lem:identifiability:XnotinParentsY} deals with the trivial case of identifiability.
Lemma~\ref{lem:blocked_when_tmax>t_or_tmin<t-gammaxy-gammamax} shows that every path outside of the time slices of $X_{t-\gamma_{xy}}$ and $Y_t$ are easily blocked.
Property~\ref{prop:primpath1} shows that passively blocked paths are not problematic for identification.
Lemma~\ref{lem:identifiability:BlocksWalksNotinDescendantsY} together with Property~\ref{prop:primpath3} states that we will always be able to block paths in which some vertices are not descendants of $Y$.
Lemmas~\ref{lem:identifiability:BlocksWhenGamma>0} and ~\ref{lem:identifiability:BlocksWhenGamma>0Andn=2} together with Properties~\ref{prop:primpath2} and ~\ref{prop:primpath4} show that in the case of positive lag (\ie, $\gamma_{xy}>0$) we can use this temporal information to block other specific paths.
Together, these lemmas give a set of sufficient conditions for the direct effect to be identifiable.
These conditions are in fact necessary and sufficient.
This is summarized in Theorem~\ref{theorem:identifiability}.

\mytheoremone*

\begin{proof}
    Firstly, let us prove the backward implication.
    Let $\gamma_{xy} \geq 0$ and $X\in Par(Y,\mathcal{G}_s)$. 
    \begin{itemize}
        \item Suppose Condition~\ref{item:1} holds.
        Let $C$ be a cycle on $X$ with $Y\notin C$ and since $X\in Par(Y,\mathcal{G}_s)$, $\pi_s=C+Y=\langle V^1,\dots,V^{n} \rangle$ is a walk in $\mathcal{G}_s$.
        If $\gamma_{xy} > 0$ then $\pi_f = \langle V^1_{t-\gamma_{xy}},V^2_t,\dots,V^{n}_t \rangle$ is a path that exists in a compatible FTCG and every set $\mathcal{Z}_f$ that blocks this path contains a descendant of $Y_t$ in another compatible FTCG.
        If $\gamma_{xy}=0$, then because $X\in Desc(Y,\mathcal{G}_s)$ one can take $\langle V^n,\dots,V^1 \rangle$ a directed path from $Y$ to $X$ and $\pi_f = \langle V^1_{t},\dots,V^{n}_t \rangle$ is a path that exists in a compatible FTCG and every set $\mathcal{Z}_f$ that blocks this path contains a descendant of $Y_t$ in another compatible FTCG.
        \item Suppose Condition~\ref{item:2} holds.
        Let $\pi_s = \langle V^1,\dots,V^n \rangle$ as described.
        If $\pi_s=X\rightleftarrows Y$ and $\gamma_{xy}=0$, then there exists a compatible FTCG $\mathcal{G}_f$ in which the path $\pi_f=X_t \leftarrow Y_t$ exists and cannot be blocked.
        Else, $n \geq 3$ and $\pi_s$ is active so there exists a compatible FTCG $\mathcal{G}_f^1$ in which the path $\pi_f = \langle V^1_{t-\gamma_{xy}},V^2_t,\dots,V^{n-1}_t,V^n_t \rangle$ exists and is active.
        Notice that there exists another compatible FTCG $\mathcal{G}_f^2$ in which $\{V^2_t,\dots,V^{n-1}_t\} \subseteq Desc(Y_t,\mathcal{G}_f^2)$ and every set $\mathcal{Z}_f$ that blocks $\pi_f$ in $\mathcal{G}_f^1$ contains a vertex in $Desc(Y_t,\mathcal{G}_f^2)$.
    \end{itemize}

    
    Lemma~\ref{lem:identifiability:XnotinParentsY} gives the first trivial condition $X \notin Par(Y,\mathcal{G}_s) \implies \alpha_{Y_t, X_{t-\gamma_{xy}}}$ identifiable, so we assume in the remaining of the proof $X \in Par(Y,\mathcal{G}_s)$.
    To prove the forward implication it suffices to prove that if we suppose that $\mathcal{G}_s$ does not verify any condition of Theorem~\ref{theorem:identifiability}, then there exists an adjustment set that: 
    \begin{enumerate}
    \item does not contain any descendant of $Y_t$ in any FTCG that is compatible with $\mathcal{G}_s$ nor $X_{t-\gamma_{xy}}$, and
    \item blocks every non-direct path from $X_{t-\gamma_{xy}}$ to $Y_t$ in every FTCG that is compatible with $\mathcal{G}_s$.
    \end{enumerate}
    Consider $\mathcal{Z}_{f}$ as defined in Definition~\ref{def:huge-single-door_set}.
    By construction $\mathcal{Z}_{f}$ does not contain any descendant of $Y_t$ nor $X_{t-\gamma_{xy}}$. 
    To show the second point, let us consider $\pi_f=\langle V^1_{t^1},\dots,V^n_{t^n} \rangle$ to be a path from $X_{t-\gamma_{xy}}$ to $Y_t$ in a FTCG $\mathcal{G}_f$ compatible with $\mathcal{G}_s$.
    Let $\phi(\pi_f)=\pi_s=\langle V^1,\dots,V^n \rangle$ be its compatible walk in $\mathcal{G}_s$ and $\pi'_s = \langle U^1,\dots,U^m \rangle$ the primary path of $\pi_s$.
    In the following, we exhaustively explore the possible characteristics of $\pi'_s$ and show how in every case either it is either direct, or $\mathcal{Z}_f$-blocked or one of the conditions of Theorem~\ref{theorem:identifiability} is verified.
    
    Firstly, suppose $m=2$.
    According to Property~\ref{prop:primpath2} this implies either $n=2$ or $\exists 1<i<n \st V^i=Y$ or $\exists 1<i<n \st V^i=X$.
    \begin{itemize}
        \item If $n=2$, then either $\pi_s$ is direct and so is $\pi_f$ or $\pi_s = X\rightleftarrows Y$.
        In this second case, if $\gamma_{xy}=0$ then Condition~\ref{item:2.1} of Theorem~\ref{theorem:identifiability} is verified and if $\gamma_{xy}>0$ then $\pi_f$ is direct as a path $X_{t-\gamma_{xy}}\leftarrow Y_t$ is not possible.
        \item If $\exists 1<i<n \st V^i=Y$, then, if $\gamma_{xy}>0$, Lemma~\ref{lem:identifiability:BlocksWhenGamma>0Andn=2} shows that $\pi_f$ is $\mathcal{Z}_f$-blocked.
        Moreover, since $\pi_f$ is a path, one knows that $t^i \neq t$, and if $\gamma_{xy}=0$, this forces $t_{max}(\pi_f)>t$ or $t_{min}(\pi_f)<t-\gamma_{xy}$ in which case Lemma~\ref{lem:blocked_when_tmax>t_or_tmin<t-gammaxy-gammamax} shows that $\pi_f$ is $\mathcal{Z}_f$-blocked.
        \item If $\forall 1<i<n, V^i\neq Y$ and $\exists 1<i<n \st V^i=X$, since $\pi_f$ is a path, one knows that $t^i \neq t-\gamma_{xy}$, and if $\gamma_{xy}=0$, this forces $t_{max}(\pi_f)>t$ or $t_{min}(\pi_f)<t-\gamma_{xy}$ in which case Lemma~\ref{lem:blocked_when_tmax>t_or_tmin<t-gammaxy-gammamax} shows that $\pi_f$ is $\mathcal{Z}_f$-blocked.
        Additionally, if $\gamma_{xy}>0$, Lemma~\ref{lem:identifiability:BlocksWalksNotinDescendantsY} shows that if $V^i = X\notin Desc(Y,\mathcal{G}_s)$ then $\pi_f$ is $\mathcal{Z}_f$-blocked and Lemma~\ref{lem:identifiability:BlocksWhenGamma>0} shows that if $\exists 1\leq j < n \st V^j \leftarrow V^{j+1}$ then  $\pi_f$ is $\mathcal{Z}_f$-blocked.
        However, if $\forall 1\leq j < n$, $V^j \rightarrow V^{j+1}$ or $V^j \rightleftarrows V^{j+1}$, then $C=\langle V^1,\dots,V^i \rangle \in Cycles(X,\mathcal{G}_s)$ and $Y\notin C$ and together with $X\in Desc(Y,\mathcal{G}_s)$ this verifies Condition~\ref{item:1} of Theorem~\ref{theorem:identifiability}.
    \end{itemize}
    Now, suppose $m\geq3$.
    \begin{itemize}
        \item If $\exists 1<i<m$ such that $U^i\notin Desc(Y,\mathcal{G}_s)$.
        Then, Property~\ref{prop:primpath3} shows that $\exists 1<j<n$ such that $V^j\notin Desc(Y,\mathcal{G}_s)$.
        In this case, Lemma~\ref{lem:identifiability:BlocksWalksNotinDescendantsY} shows that $\pi_f$ is $\mathcal{Z}_f$-blocked.
        \item If $\pi'_s$ is not active, then it is passively blocked by $U^i$ for $1<i<m$.
        One can assume $U^i\in Desc(Y,\mathcal{G}_s)$ as the case $U^i\notin Desc(Y,\mathcal{G}_s)$ was previously treated.
        Then, Property~\ref{prop:primpath3} shows that $\pi_s$ is passively blocked by $V^j \in Desc(U^i,\mathcal{G}_s) \subseteq Desc(Y,\mathcal{G}_s)$ for $1<j<n$.
        If $\gamma_{xy} = 0$, then, either $t^j \neq t$ and Lemma~\ref{lem:blocked_when_tmax>t_or_tmin<t-gammaxy-gammamax} shows that $\pi_f$ is $\mathcal{Z}_f$-blocked, either $t^j=t$ and thus $V^j_{t^j}\in \mathcal{D}_{\geq t}$ and $\pi_f$ contains a collider $V^k_{t^k} \in Desc(V^j_{t^j},\mathcal{G}_f) \subseteq \mathcal{D}_{\geq t}$ so $\pi_f$ is $\mathcal{Z}_f$-blocked.
        If $\gamma_{xy}>0$, then since $\pi_s$ is passively blocked by $V^j \in Desc(U^i,\mathcal{G}_s) \subseteq Desc(Y,\mathcal{G}_s)$, there $\exists j\leq k$ such that $V^k\leftarrow V^{k+1}$ in $\pi_s$ in which case Lemma~\ref{lem:identifiability:BlocksWhenGamma>0} shows that $\pi_f$ is $\mathcal{Z}_f$-blocked.
        \item If no previous condition is verified and $\gamma_{xy}=0$ then Condition~\ref{item:2.1} of Theorem~\ref{theorem:identifiability} is verified and if $\gamma_{xy}>0$ and $\nexists 1\leq i < m \st U^i\leftarrow U^{i+1}$ in $\pi'_s$ then Condition~\ref{item:2.2} of Theorem~\ref{theorem:identifiability} is verified.
        Therefore, the last case remaining is when $\gamma_{xy}>0$ and $\exists 1\leq i < m \st U^i\leftarrow U^{i+1}$ in $\pi'_s$.
        Property~\ref{prop:primpath4} shows that, in this case, $\exists 1\leq j < n \st V^j\leftarrow V^{j+1}$ in $\pi_s$.
        According to Lemma~\ref{lem:identifiability:BlocksWhenGamma>0}, $\pi_f$ is $\mathcal{Z}_f$-blocked.
    \end{itemize}
\end{proof}

\mypropositionone*
\begin{proof}
    Let $\mathcal{G}_s=(\mathcal{V}_s,\mathcal{E}_s)$ be a SCG, $X,Y\in \mathcal{V}_s$ with $X\in Par(Y,\mathcal{G}_s)$ and $\gamma_{xy}$ be a lag.
    Suppose the direct effect of $X_{t-\gamma_{xy}}$ on $Y_t$ is identifiable following Theorem~\ref{theorem:identifiability}.
    Let $\mathcal{Z}_{f}$ be the adjustment set relative to $(X_{t-\gamma_{xy}}, Y_t)$ defined in Definition~\ref{def:single-door_set}.
    Let $\mathcal{G}_f=(\mathcal{V}_f,\mathcal{E}_f)$ be a compatible FTCG of maximal lag at most $\gamma_{max}$ compatible with $\mathcal{G}_s$.
    \begin{itemize}
        \item Firstly, using the decomposition $\mathcal{Z}_{f} = \mathcal{D}^{Anc(Y)}_{<t} \cup \mathcal{A}^{Anc(Y)}_{\leq t}$ as in Definition~\ref{def:single-door_set}, and because $Desc(Y_t,\mathcal{G}_f) \subseteq \{V_{t'} | V\in Desc(Y,\mathcal{G}_s),~t'\geq t\}$ it is clear that $\mathcal{Z}_{f} \cap (Desc(Y_t,\mathcal{G}_f) \cup \{X_{t-\gamma_{xy}}\}) = \emptyset$.
        \item Secondly, let $\pi_f=\langle V^1_{t^1},\dots,V^n_{t^n} \rangle$ be a non-direct path from $X_{t-\gamma_{xy}}$ to $Y_t$ in FTCG $\mathcal{G}_f$ and $\pi_s=\langle V^1,\dots,V^n \rangle=\phi(\pi_f)$ its compatible walk.
        If $t_{max}(\pi_f) > t$ then $\exists 1 < k< n \st \rightarrow V^k_{t_{max}(\pi_f)} \leftarrow$ in $\pi_f$ with $Desc(V^k_{t_{max}(\pi_f)},\mathcal{G}_f) \cap \mathcal{Z}_{f} = \emptyset$ and thus $\pi_f$ is passively blocked by $\mathcal{Z}_{f}$.
        Therefore, for the following we can suppose $t_{max}(\pi_f) \leq t$.
        Because the direct effect of $X_{t-\gamma_{xy}}$ on $Y_t$ is identifiable following Theorem~\ref{theorem:identifiability} we know that $\pi_f \neq \langle V^1_{t^1} \leftarrow \dots \leftarrow V^n_{t^n} \rangle$. Therefore, $\exists k_{max} = max\{1 < k \leq n| V^{k-1}_{t^{k-1}} \rightarrow V^k_{t^k}\}$ with $t^{k_{max}} \geq t $ and $V^{k_{max}}\in Desc(Y,\mathcal{G}_s)$ and since $\pi_f$ is non-direct this forces $n\geq 3$.
        \begin{itemize}
            \item If $k_{max} < n$ then $\rightarrow V^{k_{max}}_{t^{k_{max}}} \leftarrow$ and $Desc(V^{k_{max}}_{t^{k_{max}}},\mathcal{G}_f) \cap \mathcal{Z}_{f} = \emptyset$ so $\pi_f$ is passively blocked by $\mathcal{Z}_{f}$.
            \item If $k_{max}=n$ (\ie, $V^{n-1}_{t^{n-1}} \rightarrow V^n_{t^n}$) and $\pi_f = \langle V^1_{t^1}\rightarrow \dots \rightarrow V^n_{t^n} \rangle$ then because the direct effect of $X_{t-\gamma_{xy}}$ on $Y_t$ is identifiable following Theorem~\ref{theorem:identifiability} we know that $\exists d_{max} = \max\{1 < d < n | V^d  \notin Desc(Y,\mathcal{G}_s)\}$ and $t-\gamma_{max} \leq t^{d_{max}} \leq t$, so since $V^{d_{max}}\in Anc(Y,\mathcal{G}_s) \backslash Desc(Y,\mathcal{G}_s)$, we have $V^{d_{max}}_{t^{d_{max}}} \in \mathcal{A}^{Anc(Y)}_{\leq t}$ and thus $\pi_f$ is manually blocked by $\mathcal{Z}_{f}$.
            \item If $k_{max}=n$ (\ie, $V^{n-1}_{t^{n-1}} \rightarrow V^n_{t^n}$) and $\exists l_{max} = max\{ 1 < l< n |V^{l-1}_{t^{l-1}} \leftarrow V^l_{t^l}\}$ then $\langle V^{l_{max}},\dots,V^n\rangle \subseteq Anc(Y,\mathcal{G}_s)$
            \begin{itemize}
                \item If $t^{l_{max}} < t$ then $\exists l_{max} \leq i$ such that $V^i_{t^i} \rightarrow V^{i+1}_{t^{i+1}}$ and $t-\gamma_{max} \leq t^i < t^{i+1} = t$ and since $V^{i} \in Anc(Y,\mathcal{G}_s)$, $V^{i}_{t^{i}} \in \mathcal{Z}_{f}$ and thus $\pi_f$ is manually blocked by $\mathcal{Z}_{f}$.
                \item If $V^{l_{max}} \notin Desc(Y,\mathcal{G}_s)$ and $t^{l_{max}} = t$ then, since $V^{l_{max}} \in Anc(Y,\mathcal{G}_s)$, $V^{l_{max}}_{t^{l_{max}}} \in \mathcal{A}^{Anc(Y)}_{\leq t}$ and thus $\pi_f$ is manually blocked by $\mathcal{Z}_{f}$.
                \item If $V^{l_{max}} \in Desc(Y,\mathcal{G}_s)$ and $t^{l_{max}} = t$ and $\exists r_{max} = max\{1 < r < l_{max} | V^{r-1}_{t^{r-1}} \rightarrow V^{r}_{t^{r}}\}$, then notice that $\langle \rightarrow V^{r_{max}}_{t^{r_{max}}} \leftarrow \dots \leftarrow V^{l_{max}}_{t^{l_{max}}} \rangle$ forces $V^{r_{max}} \in Desc(Y,\mathcal{G}_s)$ and $t^{r_{max}} = t$ so $\pi_f$ is passively blocked by $\mathcal{Z}_f$.
                \item If $V^{l_{max}} \in Desc(Y,\mathcal{G}_s)$ and $t^{l_{max}} = t$ and $\nexists 1 < r < l_{max} \st V^{r-1}_{t^{r-1}} \rightarrow V^{r}_{t^{r}}$, then $\pi_f=\langle X_{t-\gamma_{xy}} \leftarrow \dots \leftarrow V^{l_{max}}_{t^{l_{max}}} \rightarrow \dots \rightarrow Y_t \rangle$.
                Since $t^{l_{max}}=t$ and $t_{max}(\pi_f)\leq t$, $\forall 1 \leq i \leq n$, $t^i = t$ and in particular $\gamma_{xy}=0$ and thus because $\pi_f$ is a path, $\pi_s = \langle V^1,\dots,V^n\rangle$ is also a path.
                Moreover, since $V^{l_{max}} \in Desc(Y,\mathcal{G}_s)$,  $\forall 1 \leq i \leq n$, $V^i \in Desc(Y,\mathcal{G}_s)$.
                Therefore, $\pi_s$ verifies Condition~\ref{item:2.1} which is impossible.
            \end{itemize}
        \end{itemize}
        In conclusion, $\pi_f$ is blocked by $\mathcal{Z}_{f}$.
    \end{itemize}
\end{proof}

% ...


\bibliography{aaai24}

\end{document}